\documentclass[11pt]{article} 

\pdfoutput=1

\usepackage{graphicx} 
\usepackage{subfigure}

\usepackage{rotating}

\RequirePackage[numbers]{natbib}

\bibpunct{(}{)}{;}{a}{}{,}

\usepackage{soul}

\usepackage{fullpage}

\usepackage{algorithm}
\usepackage{authblk}
\usepackage{amsmath,bbm,amssymb,amsthm,array}
\usepackage{url}
\usepackage[utf8]{inputenc} 
\usepackage[T1]{fontenc}    
\usepackage{hyperref}       
\usepackage{url}            
\usepackage{booktabs}       
\usepackage{amsfonts}       
\usepackage{nicefrac}       
\usepackage{microtype}      

\usepackage{mathtools}
\usepackage{subfigure}
\usepackage{algorithm}
\usepackage{algorithmic}
\usepackage{enumitem}
\usepackage{multirow}
\usepackage{array}
\usepackage{bm}
\usepackage{color}
\newtheorem{lemma}{Lemma}

\newcommand{\change}[1]{\textcolor{black} {#1}}
\newcommand\inner[2]{\langle #1, \, #2 \rangle}
\newcommand\biginner[2]{\big\langle #1, \, #2 \big\rangle}
\newcommand\Biginner[2]{\Big\langle #1, \, #2 \Big\rangle}
\newcommand\doubleInner[1]{\langle \! \langle #1 \rangle \! \rangle}
\newcommand\bigdoubleInner[2]{\big\langle \! \big\langle #1, \, #2 \big\rangle \! \big\rangle}
\newcommand\BigdoubleInner[2]{\Big\langle \!\! \Big\langle #1, \, #2 \Big\rangle \!\! \Big\rangle}
\newcommand\matNorm[1]{|\!|\!|#1|\!|\!|}
\newcommand\offNorm[1]{\|#1\|_{1,\textnormal{off}}}

\def\Lw{\widetilde{w}}
\def\Tparam{\param^*}

\def\Param{\Theta}
\def\LParam{\widetilde{\Param}}
\def\TParam{\Param^*}

\def\xi{X^{(i)}}
\def\F{{\textnormal{F}}}
\newcommand{\argmin}{\operatornamewithlimits{argmin}}
\DeclareMathOperator*{\minimize}{minimize}
\DeclareMathOperator*{\sta}{s.t.}
\def\P{\mathbb{P}}
\def\reals{\mathbb{R}}
\def\R{\mathcal{R}}
\def\Ttheta{\theta^*}
\def\Tw{w^*}

\def\Ltheta{\widetilde{\theta}}
\def\Lw{\widetilde{w}}

\def\errt{\Delta}
\def\errw{\Gamma}
\def\Lerrt{\widetilde{\Delta}}
\def\Lerrw{\widetilde{\Gamma}}

\def\RSCcon{{\kappa_{l}}}
\def\RSCtolOne{\tau_1(n,p)}

\def\IncConOne{\tau_2(n,p)}
\def\IncConTwo{\tau_3(n,p)}

\def\Sig{\Sigma}
\def\SigG{\Sigma_G}
\def\SigB{\Sigma_B}

\def\eventOne{\mathcal{A}}
\def\Loss{\mathcal{L}}
\def\Lossi{\bar{\Loss}}

\def\param{\theta}
\def\tparam{\param^*}

\def\abovestrut#1{\rule[0in]{0in}{#1}\ignorespaces}
\def\belowstrut#1{\rule[-#1]{0in}{#1}\ignorespaces}

\def\abovespace{\abovestrut{0.20in}}

\def\belowspace{\belowstrut{0.10in}}

\newlength{\widebarargwidth}
\newlength{\widebarargheight}
\newlength{\widebarargdepth}
\DeclareRobustCommand{\widebar}[1]{%
  \settowidth{\widebarargwidth}{\ensuremath{#1}}%
  \settoheight{\widebarargheight}{\ensuremath{#1}}%
  \settodepth{\widebarargdepth}{\ensuremath{#1}}%
  \addtolength{\widebarargwidth}{-0.3\widebarargheight}%
  \addtolength{\widebarargwidth}{-0.3\widebarargdepth}%
  \makebox[0pt][l]{\hspace{0.15\widebarargheight}%
    \hspace{0.3\widebarargdepth}%
    \addtolength{\widebarargheight}{0.3ex}%
    \rule[\widebarargheight]{0.95\widebarargwidth}{0.1ex}}%
  {#1}}

\newtheorem{proposition}{Proposition}
\newtheorem{corollary}{Corollary}

\usepackage{hyperref}
\usepackage{url}

\newtheorem{theorem}{Theorem}

\begin{document}

\title{A General Family of Trimmed Estimators for Robust High-dimensional Data Analysis
}

\author[1]{Eunho Yang}
\author[2]{Aur\'elie C. Lozano} 
\author[3]{Aleksandr Aravkin} 

\affil[1]{School of Computing, Korea Advanced Institute of Science and Technology, South Korea}
\affil[2]{Mathematical Sciences Department, IBM T.J. Watson Research Center, USA}
\affil[3]{Department of Applied Mathematics, University of Washington, USA}


\maketitle

\begin{abstract}
  We consider the problem of robustifying high-dimensional structured estimation. 
Robust techniques are key in real-world applications which often involve outliers and data corruption. 
We focus on trimmed versions of structurally regularized M-estimators in the high-dimensional setting, including the popular Least Trimmed Squares estimator, as well as analogous estimators for generalized linear models and graphical models, using possibly non-convex loss functions. We present a general analysis of their statistical convergence rates 
and consistency, and then take a closer look at the trimmed versions of the Lasso and Graphical Lasso estimators as special cases. On the optimization side, we show how to extend algorithms for M-estimators to fit trimmed variants and provide guarantees on their numerical convergence.
The generality and competitive performance of high-dimensional trimmed estimators are illustrated numerically on both simulated and real-world genomics data.
\end{abstract}


\section{Introduction}\label{sec:intro}
We consider the problem of high-dimensional estimation, where the number of variables $p$ may greatly exceed the number of observations $n.$ Such high-dimensional settings are becoming increasingly prominent in a variety of domains, including signal processing, computational biology and finance.  
The development and the statistical analysis of structurally constrained estimators for high-dimensional estimation has recently attracted considerable attention.  These estimators seek to minimize  the sum of a loss function and a weighted regularizer. 
The most popular example is that of Lasso~\citep{Tibshirani96}, 
which solves an $\ell_1$-regularized (or equivalently $\ell_1$-constrained) least squares problem. 
Under sub-Gaussian errors, Lasso has been shown to have strong statistical guarantees~\citep{GeerBuhl09,Wainwright2006new}. 
Regularized maximum likelihood estimators (MLEs) have been developed for sparsity-structured Generalized Linear Models (GLMs), with  theoretical guarantees such as  $\ell_1$ and $\ell_2$-consistency~\citep{NRWY12}, and model selection consistency~\citep{Bunea08}. For matrix-structured regression problems, estimators using nuclear-norm regularization have been studied e.g. by~\cite{RecFazPar10}. Another prime example is that of sparse inverse covariance estimation for graphical model selection~\citep{RWRY11}.

In practice, however, the desirable theoretical properties of such regularized M-estimators can be compromised, since outliers and corruptions are often present in high-dimensional data problems. 
These challenges motivate the development of robust structured learning methods that can cope 
with observations deviating from the model assumptions. The problem of reliable high-dimensional estimation under possibly gross error has  gained increasing attention. Relevant prior work includes the ``extended'' Lasso formulation~\citep{Nguyen2011b} which incorporates an additional sparse error vector to the original Lasso problem so as to account for corrupted observations, the LAD-Lasso~\citep{Wang2007} which uses the least absolute deviation combined with an $\ell_1$ penalty, and the Robust Matching Pursuit method of~\cite{Chen13} which performs feature selection based on a trimmed inner product of the features with the residuals, rather than a full inner product, so as to alleviate the impact of corrupted observations. In general, however, extending $M$-estimators beyond the least squares case is challenging. For example, \cite{YTR13, tibshirani2014robust} extend the strategy in \cite{Nguyen2011b} to generalized linear models in two ways: the first requires modeling errors in the input space, which maintains convexity of the objective but imposes stringent conditions for consistency; the other modeling errors in the output space, breaks convexity and yield milder conditions.  
 
 A key motivation for trimmed approaches is that convex loss functions with linear tail growth
(such as the $\ell_1$-norm and Huber loss) are not robust enough. As~\cite{SparseLS} points out, both of these approaches have a breakdown point of $\epsilon = 0$, since even a single gross contamination can arbitrarily distort the regression coefficients. 
Remarkably, the median of least squares residual originally proposed by~\cite{LTS84} avoids this problem, reaching  
breakdown point of nearly 50\%; the approach is equivalent to `trimming' a portion of the largest residuals. This lead to the consideration of sparse Least Trimmed Squares (sparse LTS) for robust high-dimensional estimation. While \cite{SparseLS} established high breakdown point property for sparse LTS, its statistical convergence has not been previously analyzed.  

In this paper, we present a unified framework and statistical analysis for trimmed regularized M-estimators, generalizing the sparse least trimmed squares (Sparse LTS) estimator~\citep{SparseLS} to allow for a wide class of (possibly non-convex) loss functions as well as structured regularization.  %
Using our analysis, we derive error bounds for the sparse LTS estimator. These require less stringent conditions for estimation consistency than those of Extended Lasso. We also derive error bounds for sparse Gaussian graphical models (GGMs) as a specific example. In contrast, existing approaches for robust sparse GGMs estimation lack statistical guarantees.
 
In terms of optimization-side, we use partial minimization to extend existing optimization algorithms for M-estimators to trimmed formulations. 
An important example of the approach is a modified proximal gradient method. 
For convex M-estimators, we show that under moderate assumptions, the `trimming' is completed in finitely many steps, 
and thereafter the method reduces to a descent method for a convex problem over a fixed set of identified `inliers'.  
We use simulated data to compare with competing methods, and then apply our approach to real genomics datasets.

The manuscript is organized as follows. In Section~\ref{sec:setup} we introduce the general setup and present the family of  High-Dimensional Trimmed  estimators. The main theoretical results on their convergence and consistency are stated in Section~\ref{sec:gen}, along with corollaries for linear models and Gaussian graphical models respectively.  The partial minimization approach for optimization is described in Section~\ref{sec:opt}. Empirical results  are presented in Section~\ref{sec:exps} for simulated data and Section~\ref{sec:real} for the analysis of genomics datasets . All proofs are collected in the Appendix.

\section{A General Framework for High-Dimensional Trimmed Estimators}\label{sec:setup}
\paragraph{Motivating Example 1: Linear Regression.}
To motivate high-dimensional trimmed estimators, we start with high-dimensional linear regression. 
The real-valued observation $y_i \in \reals$ comes from the linear model
\begin{align}\label{EqnLinearModels}
	y_i= \inner{x_i}{\Ttheta} + o_i, \quad i=1,\hdots,n,
\end{align} 
where $x_i \in \reals^p$ is a covariate, the true regression parameter vector is $ \Ttheta = (\Ttheta_1, \ldots,  \Ttheta_p)^\top \in \reals^p$, and $ o_i$ is the observation noise. Since outliers are commonly present in high-dimensional data problems, we assume $p$ is substantially larger than $n$ without loss of generality.

Let $G$ be the set of ``good'' samples, and $B$ denote the set of ``bad'' samples arbitrary corrupted. We are particularly concerned with the scenario where all samples in $B$ are potentially badly (arbitrarily) corrupted.

In order to cope with observations that deviate from the true model, \citet{SparseLS} proposed sparse LTS, an $\ell_1$-penalized version of the classical least trimmed squares (LTS) estimator~\citep{LTS84} solving
\begin{align}\label{eq:SLTS}
	\minimize_\theta  \frac{1}{2h}\sum_{i=1}^h [r^2(\theta)]_{(i)} + \lambda \|\theta\|_1,
\end{align} 
where $r^2(\theta) = (r_1^2,\ldots,r_n^2)^T$ with $r_i^2 = (y_i - \inner{x_i}{\theta})^2$, and $[r^2(\theta)]_{(1)} \leq \ldots \leq [r^2(\theta)]_{(n)}$ are the order statistics of the squared residuals $r^2(\theta)$.
\citet{SparseLS} established the breakdown point of the resulting sparse LTS estimator, 
and proposed an iterative algorithm for its computation. 
At iteration $t$, the algorithm computes the Lasso solution based on the current subset $H_t$ of observations with $|H_t| = h,$ 
and constructs the next subset $H_{t+1}$ from the observations corresponding to the $h$ smallest squared residuals. 

Our starting point is the following reformulation of regularized LTS problem: 
	\begin{align}\label{EqnRobustLS}
		\minimize_{w \in \Delta_h, \theta \in \rho\mathbb{B}_1} \ &  \frac{1}{2h}\sum_{i=1}^{{n}} w_i \big(y_i - \inner{x_i}{\theta}\big)^2 + \lambda \|\theta\|_1 
	\end{align}
where $\Delta_h := \{w: w \in [0,1]^n,\ 1^Tw = h\}$ is the $h$-scaled capped unit simplex,  $\mathbb{B}_1$ is the $\ell_1$-norm ball, 
and the constraint $\theta\in \rho\mathbb{B}_1$ (or equivalently $\|\theta\|_1 \leq \rho$) ensures that the optimum of non-convex  problem~\eqref{EqnRobustLS} exists as discussed in \citet{LW15}. This constraint on $\theta$ is a theoretical safeguard, since problem \eqref{EqnRobustLS} is equivalent to the problem~\eqref{eq:SLTS} when $\rho$ is large enough.

\paragraph{A family of trimmed estimators.}
Based on the reformulation \eqref{EqnRobustLS}, we propose the family of trimmed estimators for general high-dimensional problems: given a collection of arbitrary corrupted samples $Z_1^n = \{Z_1, \hdots, Z_n\}$, and a differentiable (possibly non-convex) loss function $\Lossi$, we solve
	\begin{align}\label{EqnRobustGeneral}
		\minimize_{w \in \Delta_h, \theta\in \rho\mathbb{B}_{\R}}  f(w,\theta) := \ &  \frac{1}{h}\sum_{i=1}^{{n}} w_i \Lossi(\theta; Z_i)  + \lambda \R(\theta) 
	\end{align}
where $\R(\cdot)$ is a decomposable norm used as a regularizer \citep{NRWY12} to encourage particular low-dimensional structure of the estimator, and $\mathbb{B}_{\R}$ is the unit ball for $\R(\cdot)$ (in other words, the constraint $\theta\in \rho\mathbb{B}_{\R}$ ensures $\R(\theta) \leq \rho$). $h$ decides the number of samples (or sum of weights) used in the training. $h$ is ideally set as the number of uncorrupted samples in $G$, but practically we can tune the parameter $h$ by cross-validation.

\paragraph{Motivating Example 2: Graphical Models.}
Gaussian graphical models (GGMs) form a powerful class of statistical models for representing distributions over a set of variables~\citep{Lauritzen96}. These models employ undirected graphs to encode conditional independence assumptions among the variables, which is particularly convenient for exploring network structures. 
GGMs are widely used in variety of domains, including computational biology~\citep{OhD14}, natural language processing~\citep{Manning:99}, image processing~\citep{Woods78,Hassner78,Cross83}, statistical physics~\citep{Ising25}, and spatial statistics~\citep{Ripley81}. 

In such high-dimensional settings, sparsity constraints  are particularly pertinent for estimating GGMs, as they encourage only a few parameters to be non-zero and induce graphs with few edges.  
The most widely used estimator, the Graphical Lasso minimizes the negative Gaussian log-likelihood regularized by the $\ell_1$ norm of the entries (or the off-diagonal entries) of the precision matrix (see~\cite{YuaLin07,FriedHasTib2007,BanGhaAsp08}).
This estimator enjoys strong statistical guarantees (see e.g. \cite{RWRY11}). The corresponding optimization problem is a log-determinant program that can be solved with interior point methods~\citep{Boyd02} or by co-ordinate descent algorithms~\citep{FriedHasTib2007,BanGhaAsp08}.  
Alternatively neighborhood selection~\citep{Meinshausen06,YRAL12} can be employed to estimate conditional independence relationships separately for each node in the graph, via Lasso linear regression~\citep{Tibshirani96}. Under certain assumptions, the sparse GGM structure can still be recovered even under high-dimensional settings.

The aforementioned approaches rest on a fundamental assumption: the multivariate normality of the observations. However, outliers and corruption are frequently encountered in high-dimensional data (see e.g.~\cite{daye2012} for gene expression data). Contamination of a few observations can drastically affect the quality of model estimation. It is therefore imperative to devise procedures that can cope with observations deviating from the model assumption. Despite this fact, little attention has been paid to robust estimation of high-dimensional graphical models. Partially Relevant work includes~\cite{drton2011}, which leverages multivariate $t$-distributions for robustified inference and the EM algorithm. They also propose an alternative $t$-model which adds flexibility to the classical $t$ but requires the use of Monte Carlo EM or variational approximation as the likelihood function is not available explicitly. Another pertinent work is that of~\cite{sun2012} which introduces a robustified likelihood function. A two-stage procedure is proposed for model estimation, where the graphical structure is first obtained via coordinate gradient descent and the concentration matrix coefficients are subsequently re-estimated using iterative proportional fitting so as to guarantee positive definiteness of the final estimate.

A special case of the proposed family is that of the \emph{Trimmed Graphical Lasso} for robust estimation of sparse GGMs: 
	\begin{align}\label{EqnRobustGGM}
		\minimize_{\Param \in \Omega \cap R\mathbb{B}_1, w\in\Delta_h} \, \, & \BigdoubleInner{\Param}{\frac{1}{h}\sum_{i=1}^n w_i \xi(\xi)^\top} -\log \det (\Param)   + \lambda \offNorm{\Param} \, .
	\end{align}
Here for matrices $U \in \reals^{p \times p}$ and $V \in \reals^{p \times p}$, $\doubleInner{U,V}$ denotes the trace inner product $\textnormal{tr}(A\,B^T)$. For a matrix $U \in \reals^{p \times p}$ and parameter $a \in [1, \infty]$, $\|U\|_a$ denotes the element-wise $\ell_a$ norm, and $\|U\|_{a,\textnormal{off}}$ does the element-wise $\ell_a$ norm only for off-diagonal entries. For example, $\offNorm{U} := \sum_{i\neq j} |U_{ij}|$.

We provide statistical guarantees on the consistency of this estimator. To the best of our knowledge, this is in stark contrast with prior work on robust sparse GGM estimation (e.g. ~\cite{drton2011,sun2012}) which are not statistically guaranteed in theory.

\section{Statistical Guarantees of Trimmed Estimators} \label{sec:gen}

In this section, we provide a statistical analysis of the family of structurally regularized estimators \eqref{EqnRobustGeneral}. 
In order to simplify the notation in our theorem and its corollaries, we assume without loss of generality that the number of good samples is known  a priori and the tuning parameter $h$ in \eqref{EqnRobustGeneral} is exactly set as the genuine samples size, $|G|$. This is an unrealistic assumption, however, as long as we set $h$ smaller than $|G|$, the statements in the main theorem and its corollaries can be applied as they are. 

Noting that the optimization problem \eqref{EqnRobustGeneral} is non-convex, estimators returned by iterative methods for \eqref{EqnRobustGeneral} will be stationary points. We call $(\Ltheta,\Lw)$ a {\it local minimum} of~\eqref{EqnRobustGeneral} when
\begin{enumerate}
\item $\Ltheta$  is a local minimum  of $g_1(\theta) := f(\theta, \Lw)$ and 
\item $\Lw$ is a global minimum of $g_2(w) := f(\Ltheta, w)$. 
\end{enumerate}
These are precisely the points that are found by the algorithms developed in Section~\ref{sec:opt}. In this section, we give statistical error bounds for {\bf any} such points.  

\change{Consider any such local minimum $(\Ltheta,\Lw)$.} While we are mainly interested in the error bounds of our estimator for target parameter $\Ttheta$ (that is, $\Ltheta-\Ttheta$), we first define $\Tw$ as follows: for the index $i \in G$, $\Tw_i$ is simply set to $\Lw_i$ so that $\Tw_i - \Lw_i = 0$. Otherwise for the index $i \in B$, we set $\Tw_i = 0$. \change{Note that while $\Ttheta$ is fixed unconditionally, $\Tw$ is dependent on $\Lw$. However, $\Tw$ is fixed given $\Lw$.}

In order to guarantee bounded errors, we first assume that \change{given $(\Ltheta,\Lw)$}, the following restricted strong convexity condition \change{for $(\Ltheta,\Lw)$}  holds:
\begin{enumerate}[leftmargin=0.2cm, itemindent=1.2cm,label=\textbf{(C-$\bf{1}$)}, ref=\textnormal{(C-$1$)},start=1]
	\item {\bf (Restricted strong convexity (RSC) on $\theta$)} We overload notation and use $\Loss(\theta,w)$ to denote $\frac{1}{h}\sum_{i=1}^n w_i \Lossi(\theta; Z_i)$. Then, for any possible $\errt := \theta - \Ttheta$, the differentiable loss function $\Lossi$ satisfies 
		\begin{align*}
			\biginner{\nabla_{\theta} \Loss\big(\Ttheta + \errt,\Tw\big) - \nabla_{\theta}\Loss\big(\Ttheta,\Tw\big) }{\errt} \geq  \RSCcon \|\errt\|_2^2 - \RSCtolOne \R(\errt)^2, 
		\end{align*}
		where $\RSCcon$ is a curvature parameter, and $\RSCtolOne$ is a tolerance function on $n$ and $p$.\label{Con:rsc}
\end{enumerate}
Note that this condition is slightly different from the standard restricted strong convexity condition because of the \change{dependency on $\Tw$ and therefore on $\Lw$. {Each local optimum} has its own restricted strong convexity condition.} In case of no corruption with $\Tw_i=1$ for all $i$, this condition will be trivially reduced to the standard RSC condition, under which the standard general $M$-estimator has been analyzed (see \citet{NRWY12} for details). 

We additionally require the following condition for a successful estimation with \eqref{EqnRobustGeneral} on corrupted samples:
\begin{enumerate}[leftmargin=0.2cm, itemindent=1.2cm,label=\textbf{(C-$\bf{2}$)}, ref=\textnormal{(C-$2$)}, start=2]
	\item Consider {arbitrary local optimum $(\Ltheta,\Lw)$.} Letting $\Lerrt := \Ltheta-\Ttheta$ and $\Lerrw := \Lw - \Tw \in [0,1]^n$,
	\begin{align*}
		\biginner{ \nabla_{\theta} \Loss\big(\Ttheta + \Lerrt,\Tw + \Lerrw \big) - \nabla_{\theta} \Loss\big(\Ttheta + \Lerrt,\Tw\big) }{\Lerrt} 
		\geq   - \IncConOne \|\Lerrt\|_2 - \IncConTwo \R(\Lerrt) \, .
	\end{align*}\label{Con:inc}
\end{enumerate}

\change{\ref{Con:inc} can be understood as a structural incoherence condition between $\theta$ and $w$. This type of condition is also needed for the guarantees of extended LASSO \citep{Nguyen2011b} and other dirty statistical models with more than a single parameter \citep{YR13}.
Note again that due to the dependency on $\Lw$, {each local optimum} will have its \emph{own} conditions  \ref{Con:rsc} and \ref{Con:inc}. } We will see later in this section that these two conditions are mild enough for the popular estimators (such as linear models and GGMs) to satisfy.
 
Armed with these conditions, we state the main theorem on the error bounds of~\eqref{EqnRobustGeneral}:
\begin{theorem}\label{ThmGeneral}
	Consider an $M$-estimator from \eqref{EqnRobustGeneral} with \change{{\bf any} local minimum $(\Ltheta,\Lw)$, and suppose that it satisfies the conditions \ref{Con:rsc} and \ref{Con:inc}.} Suppose also that the regularization parameter $\lambda$ in \eqref{EqnRobustGeneral} is set as 
	\begin{align}\label{eq:mylam}
		\lambda \geq 4 \max \Big\{ & \R^*\Big( \nabla_{\theta}\Loss\big(\Ttheta,\Tw\big) \Big)\, , \, 2 \rho \RSCtolOne + \IncConTwo  \Big\} 
	\end{align}
	where $\R^*(v)$ is the dual norm of $\R(\cdot)$: $\sup_{u\in \reals^{p} \setminus \{0\} } \frac{\inner{u}{v}}{\R(u)}$. Then the following error bounds for $\Ltheta$ are guaranteed for a given model space $\mathcal{M}$:
	\begin{align*}
		&\|\Ltheta - \Ttheta\|_2 \leq \,  \frac{1}{\RSCcon}\Big( \frac{3\lambda \Psi }{2}  + \IncConOne \Big) \,  \quad \text{and} \quad \R\big(\Ltheta - \Ttheta\big) \leq \,  \frac{2}{\lambda\, \RSCcon}\Big( 2\lambda \Psi + \IncConOne\Big)^2 ,
	\end{align*}
	where 
	\[
	\Psi := \sup_{u \in \mathcal{M}\setminus \{0\}} R(u)/\|u\|_{2}
	\]	
	measures the compatibility between $\R(\cdot)$ and $\ell_2$ norms. 
\end{theorem}

For sparse vectors, $\Psi := \sup_{u \in \mathcal{M}\setminus \{0\}}\|u\|_1/\|u\|_2 = \sqrt{k}$ where $k$ is the sparsity of true parameter $\Ttheta$, and $\mathcal{M}$ is the space of vectors with the correct support set~\citep{NRWY12}.

The statement in Theorem \ref{ThmGeneral} is applicable to any local minimum of~\eqref{EqnRobustGeneral}, and it holds deterministically.  Probabilistic statements come in when the condition on $\lambda_n$ specified in Theorem \ref{ThmGeneral} is satisfied.  In \eqref{eq:mylam}, $\lambda$ is chosen based on $\R^*\big( \nabla_{\theta}\Loss(\Ttheta,\Tw) \big)$ similarly to \citet{NRWY12}. We shall see that the remaining terms with tolerance functions $\tau$ in \eqref{eq:mylam} have the same order as $\R^*\big( \nabla_{\theta}\Loss(\Ttheta,\Tw) \big)$ for the specific cases of linear models and GGMs developed in the next sections.

\subsection{Statistical Guarantees of High-Dimensional Least Trimmed Squares} \label{Sec:LStheory}
 
We now focus on the special case of high-dimensional linear regression, and apply Theorem \ref{ThmGeneral} to problem \eqref{EqnRobustLS}.  In particular, 
if $i \in G$,
	$y_i = \inner{x_i}{\Ttheta} + \epsilon_i$ 
where the observation noise $\epsilon_i$ follows zero mean and has sub-Gaussian tails. Otherwise, for $i \in B$,
	$y_i = \inner{x_i}{\Ttheta} + \delta_i$
where $\delta_i$ is the amount of arbitrary corruption. 

%
%
%
In order to derive an actual bound from the general framework of Theorem \ref{ThmGeneral}, we consider the following natural setting, which has been widely studied in past work on conventional high dimensional linear models: 
\begin{enumerate}[leftmargin=0.5cm, itemindent=2cm,label=\textbf{(LTS1)}, ref=\textnormal{(LTS1)}]
	\item {\bf ($\Sigma$-Gaussian ensemble)} Each sample $x_i$ is i.i.d. sampled from $N(0,\Sig)$. \label{ConLS:XGauss}
\end{enumerate}
\vspace{-.3cm}
\begin{enumerate}[leftmargin=0.5cm, itemindent=2cm,label=\textbf{(LTS2)}, ref=\textnormal{(LTS2)}]
	\item {\bf (Sub-Gaussian noise)} The noise vector $\epsilon \in \reals^n$ is zero-mean and has sub-Gaussian tails, which means that for any fixed vector $v$ such that $\|v\|_2 = 1$, 
	$\P\left[|\inner{v}{\epsilon}| \geq t\right] \leq 2 \exp \left(-\frac{t^2}{2\sigma^2}\right)$ for all $t >0$.
	The sub-Gaussian is quite a wide class of distributions, and contains the Gaussian family as well as and all bounded random variables.  
\label{ConLS:NsubGauss}
\end{enumerate}
\vspace{-.3cm}
\begin{enumerate}[leftmargin=0.5cm, itemindent=2cm,label=\textbf{(LTS3)}, ref=\textnormal{(LTS3)}]
	\item {\bf (Column normalization)} Let $X \in \reals^{n \times p}$ be the design matrix whose $i$-th row is the covariate $i$-th sample: $x_i^\top$, and $X^j\in \reals^n$ be the $j$-th column vector of $X$. Then, $\frac{\|X^j\|_2}{\sqrt{h}} \leq 1$. As pointed out in \citet{NRWY12}, we can always rescale linear models with out loss of generality to satisfy this condition. \label{ConLS:colNorm}
\end{enumerate}
The following assumptions are required for our estimator to be resilient to outliers and strongly consistent:
\begin{enumerate}[leftmargin=0.5cm, itemindent=1.8cm,label=\textbf{(C-$\bf{h}$)}, ref=\textnormal{(C-$h$)}]
	\item Let $h$ be the number of good samples: $|G| = h$ and hence $|B| = n-h$. Then, we assume that larger portion of samples are genuine and uncorrupted so that $\frac{|G|-|B|}{|G|} \geq \alpha$ where $0 < \alpha \leq 1$. If we assume that 40\% of samples are corrupted, then $\alpha = 1/3$. \label{Con:h}  
\end{enumerate}
\begin{enumerate}[leftmargin=0.5cm, itemindent=2cm,label=\textbf{(LTS4)}, ref=\textnormal{(LTS4)}]
	\item \change{We set the tuning parameter $\rho$ in \eqref{EqnRobustLS} as $\rho \leq \frac{C_1}{2} \sqrt{\frac{h}{\log p}}$ for some constant $C_1$. This setting requires that the number of good samples $h$ is larger than or equal to $\big(\frac{2k\|\Ttheta\|_\infty}{C_1}\big)^2 \log p$ so that the true regression parameter $\Ttheta$ is feasible for the objective.}\label{ConLS:R}
\end{enumerate}
%
Under these conditions, we can recover the following error bounds of high-dimensional LTS \eqref{EqnRobustLS}, as a corollary of Theorem \ref{ThmGeneral}:
\begin{corollary}\label{Cor:LS}
	Consider corrupted linear models \eqref{EqnLinearModels} when $\|\Ttheta\|_0 \leq k$. Suppose that  conditions \ref{Con:h}, \ref{ConLS:XGauss}, \ref{ConLS:NsubGauss}, \ref{ConLS:colNorm}, and \ref{ConLS:R} hold. \change{Also suppose that we  find a local minimum $(\Ltheta,\Lw)$ of \eqref{EqnRobustLS},  choosing }
	\begin{align*}
		\lambda = c \sqrt{\frac{\log p}{h}}
	\end{align*}
	where $c$ is some constant dependent on $\Sig$, $\sigma$ and the upper bound of $\frac{(\max_{i}\delta_i^2) |B|}{h}$.\footnote{Here without loss of generality, we can assume that $\frac{(\max_{i}\delta_i^2) |B|}{h}$ is bounded by some constant since we can always rescale the linear models properly without changing the signal $\theta$.} \change{Then, $(\Ltheta,\Lw)$ is guaranteed to satisfy \ref{Con:rsc} and \ref{Con:inc} for the specific case of \eqref{EqnRobustLS}, and have the following error bounds:} for some constant $c'$ depending on $c$, $\Sigma$ and the portion of genuine samples $\alpha$ in \ref{Con:h}, and some constant $c''$ smaller than $1$, 
\begin{align*}
	\|\Ltheta - \Ttheta\|_2 \leq \, & c' \bigg( \sqrt{\frac{k \log p}{h}} + c''\sqrt{\frac{|B|\log p}{h}} \bigg) \, , \nonumber\\
	\|\Ltheta - \Ttheta\|_1 \leq \, & 4c'\bigg( \sqrt{\frac{k \log p}{h}} + c'' \sqrt{\frac{|B|\log p}{h}}\bigg)^2 
\end{align*}
with probability at least $1-c_1 \exp(-c_1' h \lambda^2)$ for some universal positive constants $c_1$ and $c'_1$.
\end{corollary}


\change{Note that Corollary \ref{Cor:LS}  concerns any \emph{single} {local minimum. 
For the guarantees of \emph{multiple} local optima simultaneously, we may use a union bound from the corollary.}}

\paragraph{Remarks.}
It is instructive to compare the error rates and conditions in Corollary \ref{Cor:LS} with statistical guarantees of extended Lasso analyzed in \citet{Nguyen2011b}. The extended Lasso estimator solves:
\begin{align*}
	\minimize_{\theta,e} \ &  \frac{1}{2n}\sum_{i=1}^{{n}} \big(y_i - \inner{x_i}{\theta} - e_i \big)^2 + \lambda_\theta \|\theta\|_1 + \lambda_e \|e\|_1 
\end{align*}
where $\lambda_e$ is the regularization parameter for parameter $e$ capturing corruptions.  $e$ is encouraged to be sparse to reflect the fact that only a fraction of samples is corrupted.
The $\ell_2$ norm-based error rate in Corollary \ref{Cor:LS} is almost the same as that of extended Lasso: $\|\widehat{\theta}_{\text{E\_Lasso}} - \Ttheta\|_2 = O\Big(\sqrt{\frac{k\log p}{n}} + \sqrt{\frac{|B| \log n}{n}}\Big)$ under the standard Gaussian design setting \ref{ConLS:XGauss}. As long as at least a linear fraction of samples is not contaminated (that is, $h \geq \alpha n$ for $\alpha \in (0,1] $), $1/h \leq 1/(\alpha n)$ the error rates for both estimators will be asymptotically the same.  

However, it is important to revisit the conditions required for the statistical guarantees of extended Lasso. Besides an extended version of the restricted eigenvalue condition, \citet{Nguyen2011b} assumes a mutual incoherence condition, which in turn requires $c \sqrt{\matNorm{\Sig}_2}\max\Big\{\frac{k}{|B|},\frac{|B|}{k}\Big\} \Big(\sqrt{\frac{k}{n}} + \sqrt{\frac{|B|}{n}} + \sqrt{\frac{\log p}{n}}\Big) \leq \frac{1}{16}$ for some large and fixed constant $c$. Provided that $k$ and $|B|$ are fixed, the inequality can hold for a large enough sample size $n$. However, when $|B|$ grows with $n$, this condition will be violated; for example if (i) a square root fraction of samples is corrupted ($|B| = \alpha \sqrt{n}$) for a fixed $k$ or (ii) a linear fraction of $n$ is corrupted ($|B| = \alpha n$), then $c' \sqrt{\matNorm{\Sig}_2}$ can easily exceed $1/16$. Our experimental results of Section~\ref{sec:exps} will confirm this observation: as the fraction of corruptions increases, the performance of extended Lasso deteriorates compared to that of our estimator \eqref{EqnRobustLS}.




\paragraph{Statistical Guarantees When Covariates Are Corrupted.}
In the linear model \eqref{EqnLinearModels}, corruption is considered in the space of the response variable $y_i \in \reals$: namely an additional random variable $\delta_i \in \reals$ is used to model corruption in the response space. Even in the case where we have outliers with corrupted covariates $x_i + \delta' \in \reals^p$, $\delta_i$ can be understood as the mean-shift variable to model $\inner{\delta'}{\Ttheta}$. For linear models, modeling outliers in the parameter space or modeling them in the output space is thus equivalent (In constrast, for more general GLM settings, the link function is not the identity function and both approaches are distinct, see e.g. \citep{YTR13}). Nevertheless, when outliers stem from corrupted covariates, condition \ref{ConLS:XGauss} might be violated. 
For this setting, we introduce the following alternative condition:
\begin{enumerate}[leftmargin=0.5cm, itemindent=2cm,label=\textbf{(LTS5)}, ref=\textnormal{(LTS5)}]
	\item {\bf ($\Sigma$-Gaussian ensemble)} Each sample $x_i$ in $G$ is i.i.d. sampled from $N(0,\SigG)$. Let $X^B$ be the sub-design matrix in $\reals^{|B|\times p}$ corresponding to outliers. Then, we define $f(X^B)$ such that $\matNorm{X^B}_2 \leq f(X^B) \sqrt{|B| \log p}$. \label{ConLS:XCor}
\end{enumerate}
Under condition \ref{ConLS:XCor} we recover results similar to Corollary \ref{Cor:LS}:
\begin{corollary}\label{Cor:XCor}
	Consider linear models in \eqref{EqnLinearModels} where $\|\Ttheta\|_0 \leq k$. Suppose that all the conditions \ref{Con:h}, \ref{ConLS:NsubGauss}, \ref{ConLS:colNorm}, \ref{ConLS:R} and \ref{ConLS:XCor} hold. Also suppose that we choose the regularization parameter 
	\begin{align*}
		\lambda = c \sqrt{\frac{\log p}{h}}
	\end{align*}
	where $c$ is some constant dependent on $\SigG$, $f(X^B)$ and $\sigma$ and the upper bound of $\frac{(\max_{i}\delta_i^2) |B|}{h}$.  \change{Then, $(\Ltheta,\Lw)$ is guaranteed to have the following error bounds as before:} for some constant $c'$ depending on $c$, $\SigG$ and the portion of genuine samples $\alpha$ in \ref{Con:h}, and some constant $c''$ smaller than $1$, 
\begin{align*}
	\|\Ltheta - \Ttheta\|_2 \leq \, & c' \bigg( \sqrt{\frac{k \log p}{h}} + c''\sqrt{\frac{|B|\log p}{h}} \bigg) \, , \nonumber\\
	\|\Ltheta - \Ttheta\|_1 \leq \, & 4c'\bigg( \sqrt{\frac{k \log p}{h}} + c'' \sqrt{\frac{|B|\log p}{h}}\bigg)^2 
\end{align*}
with probability at least $1-c_1 \exp(-c_1' h \lambda^2)$ for some universal positive constants $c_1$ and $c'_1$.
\end{corollary}

\subsection{Statistical Guarantees of Trimmed Graphical Lasso} \label{Sec:GGMtheory}

We now focus on Gaussian graphical models and provide the statistical guarantees of our Trimmed Graphical Lasso estimator as presented in Section~\ref{sec:setup} (Motivating Example 2).  Our theory in this section provides the statistical error bounds on {\bf any} {local minimum of~\eqref{EqnRobustGGM}. We use $\|U\|_\F$ and $\matNorm{U}_2$ to denote the Frobenius  and spectral norms, respectively. 

Let $X = (X_1, X_2, \hdots, X_p)$ be a zero-mean Gaussian random field parameterized by $p\times p$ concentration matrix $\TParam$:
\begin{align}\label{EqnGRF}
	\P( X; \TParam) =  \exp \Big( -\frac{1}{2} \doubleInner{\TParam,XX^\top}  - A(\TParam) \Big)
\end{align}  
where $A(\TParam)$ is the log-partition function of Gaussian random field. Here, the probability density function in \eqref{EqnGRF} is associated with $p$-variate Gaussian distribution, $N(0,\Sigma^*)$ where $\Sigma^* = (\TParam)^{-1}$.


We consider the case where the number of random variables $p$ may be substantially larger than the number of sample size $n$, however, the concentration parameter of the underlying distribution is sparse so that the number of non-zero off-diagonal entries of $\Tparam$ is at most $k$: $| \{ \TParam_{ij} \, : \, \TParam_{ij} \neq 0 \text{ for } i\neq j\} | \leq k$.

We now investigate how easily we can satisfy the conditions in Theorem \ref{ThmGeneral}. Intuitively it is impossible to recover true parameter by weighting approach as in \eqref{EqnRobustGGM} when the amount of corruptions exceeds that of normal observation errors. 

To this end, suppose that we have some upper bound on the corruptions:
\begin{enumerate}[leftmargin=0.5cm, itemindent=2cm,label=\textbf{(TGL1)}, ref=\textnormal{(TGL1)}]
	\item For some function $f(\cdot)$, we have $\big(\matNorm{X^B}_2\big)^2 \leq f(X^B) \sqrt{h \log p}$ \label{Con:1}
\end{enumerate}
where $X^B$ denotes the sub-design matrix in $\reals^{|B|\times p}$ corresponding to outliers. 
Under this assumption, we can recover the following error bounds of Trimmed Graphical Lasso \eqref{EqnRobustGGM}, as a new corollary of Theorem \ref{ThmGeneral}:
\begin{corollary}\label{Cor1}
	Consider corrupted Gaussian graphical models with conditions \ref{Con:h} and \ref{Con:1}. Suppose that we compute the local optimum $(\LParam, \Lw)$ of \eqref{EqnRobustGGM} 
	choosing
		\begin{align*}
		\lambda = 4 \max \left\{ 8 (\max_i \Sigma^*_{ii}) \sqrt{\frac{30 \log p }{h - |B|}} + \frac{|B|}{h}\| \Sigma^*\|_{\infty} 
\, , \, f(X^B)\sqrt{\frac{\log p}{h}} \right\} \leq \frac{c_1 - f(X^B) \sqrt{\frac{|B|\log p}{h}}}{3R} \, .
	\end{align*}
	Then, $(\Ltheta,\Lw)$ is guaranteed to satisfy \ref{Con:rsc} and \ref{Con:inc} for the specific case of \eqref{EqnRobustGGM} and have the error bounds of
	\begin{align}\label{EqnMainBounds}
		&\|\LParam-\TParam\|_\F  \leq \frac{1}{\RSCcon}\left(\frac{3\lambda \sqrt{k+p}}{2} + f(X^B) \sqrt{\frac{|B|\log p}{h}}\right) \quad \text{and} \nonumber\\
		&\offNorm{\LParam-\TParam}  \leq \frac{2}{\lambda\, \RSCcon}\left( 3\lambda \sqrt{k+p} + f(X^B) \sqrt{\frac{2|B|\log p}{n}}\right)^2 
	\end{align}
	with probability at least $1-c_2 \exp(-c_2' h \lambda^2)$ for some universal positive constants $c_2$ and $c'_2$.
\end{corollary}

In Corollary \ref{Cor1}, the term $\sqrt{k+p}$ captures the relation between element-wise $\ell_1$ norm and the error norm $\|\cdot\|_\F$ including \emph{diagonal entries}.

If we further assume that the number of corrupted samples scales with $\sqrt{n}$ at most : 
\begin{enumerate}[leftmargin=0.5cm, itemindent=2cm,label=\textbf{(TGL2)}, ref=\textnormal{(TGL2)}]
	\item $|B| \leq a\sqrt{n}$ for some constant $a\geq0$, \label{Con:2}
\end{enumerate}
then we can derive the following result as another corollary of Theorem \ref{ThmGeneral}:
\begin{corollary}\label{Cor2}
	Consider corrupted Gaussian graphical models, and compute the local minimum $(\LParam, \Lw)$ of \eqref{EqnRobustGGM}, 
	setting 
	$$\lambda=c\sqrt{\frac{\log p}{n}}, \quad c :=  4\max \big\{16 (\max_i \Sigma^*_{ii}) \sqrt{15} + \frac{2a\| \Sigma^*\|_{\infty}}{\sqrt{\log p}} \, , \, \sqrt{2} f(X^B) \big\}.$$
	Suppose that the conditions \ref{Con:h}, \ref{Con:1} and \ref{Con:2} hold.  Then, if the sample size $n$ is lower bounded as 
	\begin{align*}
		n \geq \max \left\{16 a^2 \, , \, \big(\matNorm{\TParam}_2 + 1\big)^{4} \Big(3Rc + f(X^B)\sqrt{2|B|} \Big)^2(\log p) \right\} \, , 
	\end{align*}
	then  $(\LParam, \Lw)$ is guaranteed to satisfy \ref{Con:rsc} and \ref{Con:inc} for the specific case of \eqref{EqnRobustGGM} and have the following error bound:
	\begin{align}\label{EqnResultCor2}
		&\|\LParam-\TParam\|_\F  \leq \frac{1}{\RSCcon}\left(\frac{3c}{2} \sqrt{\frac{(k+p)\log p}{n}} + f(X^B) \sqrt{\frac{2|B|\log p}{n}} \right) \, 
	\end{align}
	with probability at least $1-c_1 \exp(-c_1' h \lambda^2)$ for some universal positive constants $c_1$ and $c'_1$.
\end{corollary}
Note that an $\offNorm{\cdot}$-norm error bound can also be easily derived using the selection of $\lambda$ from \eqref{EqnMainBounds}. 

\paragraph{Remarks.} Corollary \ref{Cor2} reveals an interesting result: even when $O(\sqrt{n})$ samples out of total $n$ samples are corrupted, our estimator \eqref{EqnRobustGGM} can successfully recover the true parameter with guaranteed error in \eqref{EqnResultCor2}. The first term in this bound is $O\Big(\sqrt{\frac{(k+p)\log p}{n}}\Big)$ which exactly matches the Frobenius error bound for the case without outliers (see \citet{RWRY11,Loh13} for example). Due to the outliers, the performance degrades with the second term, which is $O\Big(\sqrt{\frac{|B|\log p}{n}}\Big)$. To the best of our knowledge, our results are the first statistical error bounds available in the litterature on parameter estimation for Gaussian graphical models with outliers. 

\paragraph{When Outliers Follow a Gaussian Graphical Model.} Now let us provide a concrete example and show how $f(X^B)$ in \ref{Con:1} is precisely specified in this case: 
\begin{enumerate}[leftmargin=0.5cm, itemindent=2cm,label=\textbf{(TGL3)}, ref=\textnormal{(TGL3)}]
	\item Outliers in the set $B$ are drawn from another Gaussian graphical model \eqref{EqnGRF} with a parameter $(\SigB)^{-1}$. \label{Con:3}
\end{enumerate}
This can be understood as a Gaussian mixture model where most of the samples are drawn from $(\TParam)^{-1}$ which we want to estimate, and a small portion of samples are drawn from $\SigB$. In this case, Corollary \ref{Cor2} can be further shaped as follows:
\begin{corollary}\label{Cor3}
	Suppose that the conditions \ref{Con:h}, \ref{Con:2} and \ref{Con:3} hold. Then the statement in Corollary \ref{Cor2} holds with $f(X^B) := \frac{4\sqrt{2} a \big(1+\sqrt{\log p}\big)^2 \matNorm{\SigB}_2}{\sqrt{\log p}}$. 
\end{corollary}

\section{Optimization for Trimmed Estimators} \label{sec:opt}
\def\AlgoSize{}

While the objective function $f(w,\theta)$ in~\eqref{EqnRobustGeneral} is non-convex in $(w,\theta)$, 
it simplifies for block $w$ or $\theta$ held fixed. 
Perhaps for this reason, prior algorithms for trimmed approaches \citep{LTS84,SparseLS} alternated between solving for $\theta$ and $w$.
Unfortunately, each solve in $\theta$ is as expensive as finding the original (untrimmed) estimator. 

Here, we take advantage of the fact that the computational complexity of the two subproblems in $\theta$ and $w$ are completely different. 
With $w$ fixed, the problem in $\theta$ is equivalent to classic high-dimensional problems, e.g. Lasso, 
which is typically solved by first order methods. 
%
%
In contrast, the problem in $w$ for fixed $\theta$ is the simple linear program
\begin{equation}\label{EqnComputeW}
\begin{aligned}
\minimize_{w \in \Delta_h}  \ &  \sum_{i=1}^n w_i \Lossi(\theta; Z_i) 
\end{aligned}
\end{equation}
with all dependence on the predictors  captured by the current losses $\Lossi(\theta; Z_i)$. 
The solution is obtained setting $w_i = 1$ for the $h$ smallest values of $\Lossi(\theta; Z_i)$,
and setting remaining $w_i$ to $0$.

We exploit structure, using partial minimization. 
Similar ideas have been used for optimizing a range of nonlinear least squares problems~\citep{Golub2003} 
as well as more general problems involving nuisance parameters~\citep{AravkinVanLeeuwen2012}.
Rather than an alternating scheme (similar to that of \cite{SparseLS} for least squares) where we solve \emph{multiple} `weighted'  regularized problems to completion, we can rewrite the problem as follows: 

\begin{align}\label{EqnRobustGeneral3}
	\minimize_{\theta \in \rho\mathbb{B}_{\R}} \widetilde{\mathcal{L}}(\theta)  + \lambda \R(\theta), \quad \widetilde{\mathcal{L}}(\theta)  :=  
	\min_{w \in \Delta_h} \frac{1}{h} \sum_{i=1}^n w_i\Lossi(\theta; Z_i) =\frac{1}{h} \sum_{i=1}^n w_i(\theta)\Lossi(\theta; Z_i)  . 
\end{align}
Problem~\eqref{EqnRobustGeneral3} is equivalent to~\eqref{EqnRobustGeneral}.  The reader can verify that $\widetilde{\mathcal{L}}(\theta)$ is {\it not smooth}\footnote{When $h=1$, trimming equates to minimizing the {\it minimum} of $\mathcal{L}_i$, a problem which is nonsmooth and nonconvex.}. 
However, partial minimization provides a way to modify any descent method for fitting an M-estimator 
to bear on the corresponding trimmed estimator~\eqref{EqnRobustGeneral3}. 
Algorithm~\ref{alg:trimmedGeneral} gives a description of the steps involved for the specific case of extending proximal gradient descent. 
The algorithm uses the proximal mapping, 
which for the case of $\ell_1$ regularization is the soft-thresholding operator defined as $[S_\nu(u)]_i = \mbox{sign}(u_i) \max(|u_i|-\nu, 0)$. We assume that we pick $\rho$ sufficiently large, so one does not need to enforce the constraint $\R(\theta) \leq \rho$ explicitly.  

\begin{algorithm}[t]
	\caption{\label{algo:alt}\AlgoSize Partial Minimization using Proximal Gradient Descent for~\eqref{EqnRobustGeneral3}}
	\label{alg:trimmedGeneral}
	\begin{algorithmic}
		\AlgoSize
		\STATE Initialize $\theta^{(0)}$, $t=0$ 
		\REPEAT
			\STATE Compute $w^{(t)}$ given $\theta^{(t)}$ as the \emph{global} minimum of \eqref{EqnComputeW}
			\STATE Given $w^{(t)}$, compute the direction $\mathcal{G}^{(t+1)} \leftarrow \frac{1}{h}\sum_{i=1}^n w^{(t)}_i \nabla_{\theta} \Lossi(\theta^{(t)};y_i,x_i)$
			\STATE Update $\theta^{(t+1)} \leftarrow \mathcal{S}_{\eta^{(t+1)} \lambda}(\theta^{(t)} - \eta^{(t+1)}  \mathcal{G}^{(t+1)})$, with $\eta^{(t)}$ selected using line search.
		\UNTIL{stopping criterion is satisfied}
	\end{algorithmic}
\end{algorithm}

When the loss~$\mathcal{L}$ is convex and smooth with Lipschitz continuous gradient, the proximal gradient has a global convergence theory (see e.g. \cite{nest_lect_intro}). Convergence 
of the extended Algorithm \ref{alg:trimmedGeneral} is analyzed in the following proposition.  
%
\begin{proposition}\label{ThmOpt}
	Consider any monotonic algorithm $\mathcal{A}$ for solving $\widehat{\theta} \in \argmin_\theta F(\theta):= \Loss(\theta) + \lambda \R(\theta)$, i.e. (i) $\mathcal{A}$ guarantees that $F(\theta^{k+1}) \leq F(\theta^k)$
and  (ii) for any fixed $w \in \Delta_h$, $\mathcal{A}$ produces converging sequence of $\{\theta^{(t)}\}$ when solving $\argmin_\theta F(\theta ; w) := f(w,\theta)$. 
	If $\mathcal{A}$ is extended to solve \eqref{EqnRobustGeneral3} using partial minimization~\eqref{EqnComputeW}, 
	the monotonic property is preserved, at least one limit point exists, and 
	 every limit point of the sequence $\{(\theta^{(t)},w^{(t)})\}$ is a stationary point of~\eqref{EqnRobustGeneral}. Moreover, if $F$ is convex, 
	 and estimators over each feasible data selection have different optimal values, 
	 then $w^{(t)}$ converge in finitely many steps, and the extended algorithm converges to a local minimum\footnote{$\Ltheta$  is a local minimum  of $g_1(\theta) := f(\theta, \Lw)$ and  $\Lw$ is a global minimum of $g_2(w) := f(\Ltheta, w)$.} of~\eqref{EqnRobustGeneral}.

\end{proposition}
	
Finite convergence of the weights $w^{(t)}$ is an important point for practical implementation, since once 
the weights converge, one is essentially solving a single estimation problem, rather than 
a sequence of such problems. In particular, after finitely many steps, the extended algorithm inherits 
all properties of the original algorithm $\mathcal{A}$ for the M-estimator over the selected data.

\def\AlgoSize{}

\section{Simulated Data Experiments}\label{sec:exps}

We illustrate the generality of our approach by considering sparse logistic regression, trace-norm regularized multi-reponse regression and sparse GGMs  (For experiments with sparse linear models, see~\citet{SparseLS}).

\subsection {Simulations for Sparse Logistic Regression}
We begin with sparse logistic regression. We adopt an experimental protocol similar to~\cite{YTR13}. We consider $p=200$ features. The parameter vectors have  $k=\sqrt p$ non-zero entries sampled i.i.d. from $N(0,1).$  The data matrix $X$ is such that each of its $n$ observations is sampled from a standard Normal distribution $N(0,I_p).$  Given each observation, we draw a true class label from $\{0,1\}$ following the logistic regression model.  We show two scenarios, selecting either $\sqrt n$ or $0.1n$ samples with the highest amplitude of $\langle \theta^*, x_i \rangle$ and flipping their labels. We compare the $\ell_2$ errors over 100 simulation runs of the new estimator with those of vanilla Lasso for logistic regression, and with two extended Lasso methods for logistic regression of~\cite{YTR13} (with ``error in parameter'' and in ``error in output'') as the sample size $n$ increases. Figure~\ref{fig:logit} shows that the trimmed approach has both better performance (achieves lower errors), and is faster, matching the computational efficiency of the vanilla Lasso method. This result is anticipated by Proposition~\ref{ThmOpt}:  the weights $w^{(t)}$ converge in finitely many steps, and then we are essentially solving the Lasso with a fixed weight set thereafter. 

\begin{figure}[t]
	\centering
	\subfigure{\includegraphics[width=0.3\textwidth]{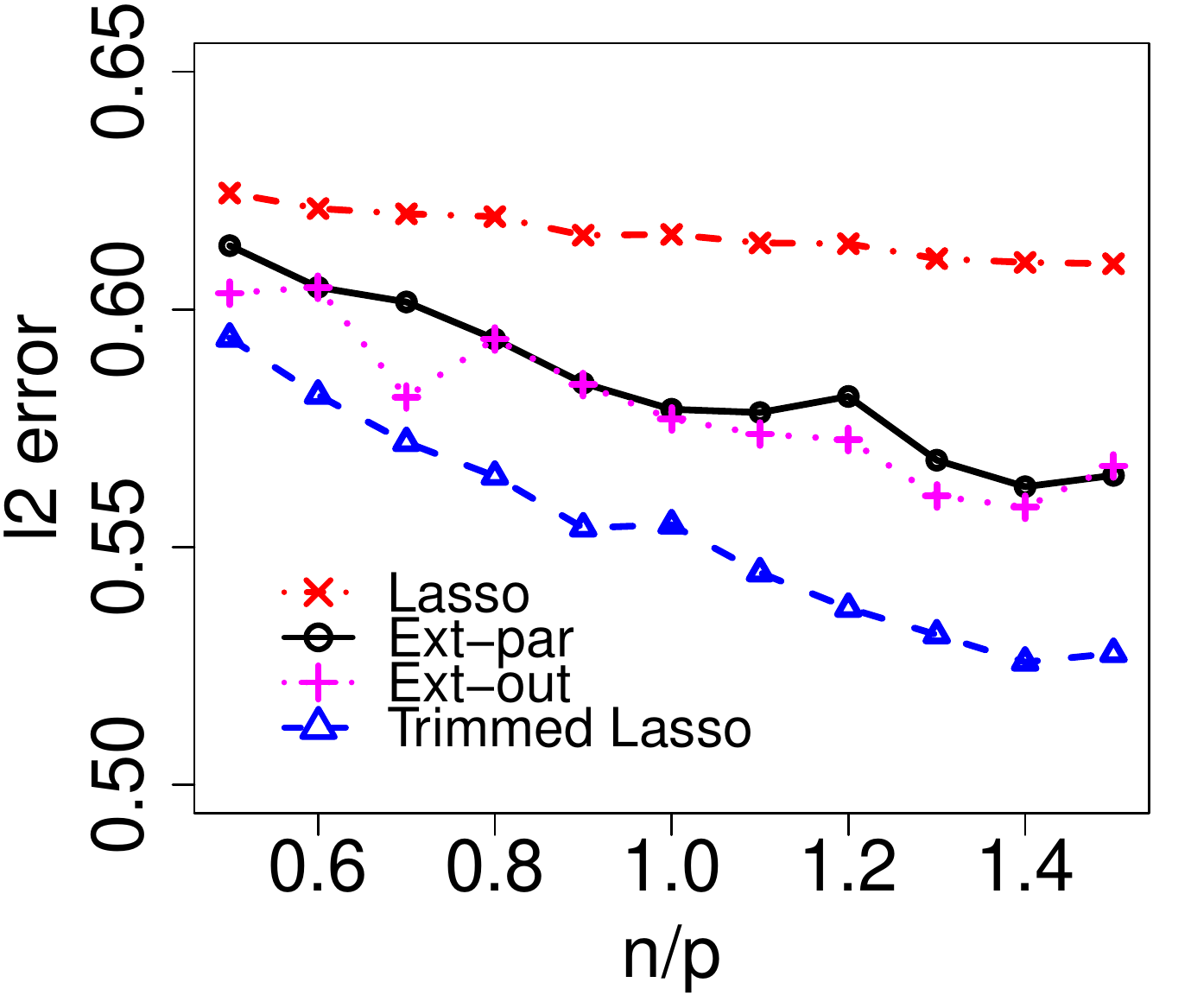}}
	\hspace{.5cm}
	\subfigure{\includegraphics[width=0.3\textwidth]{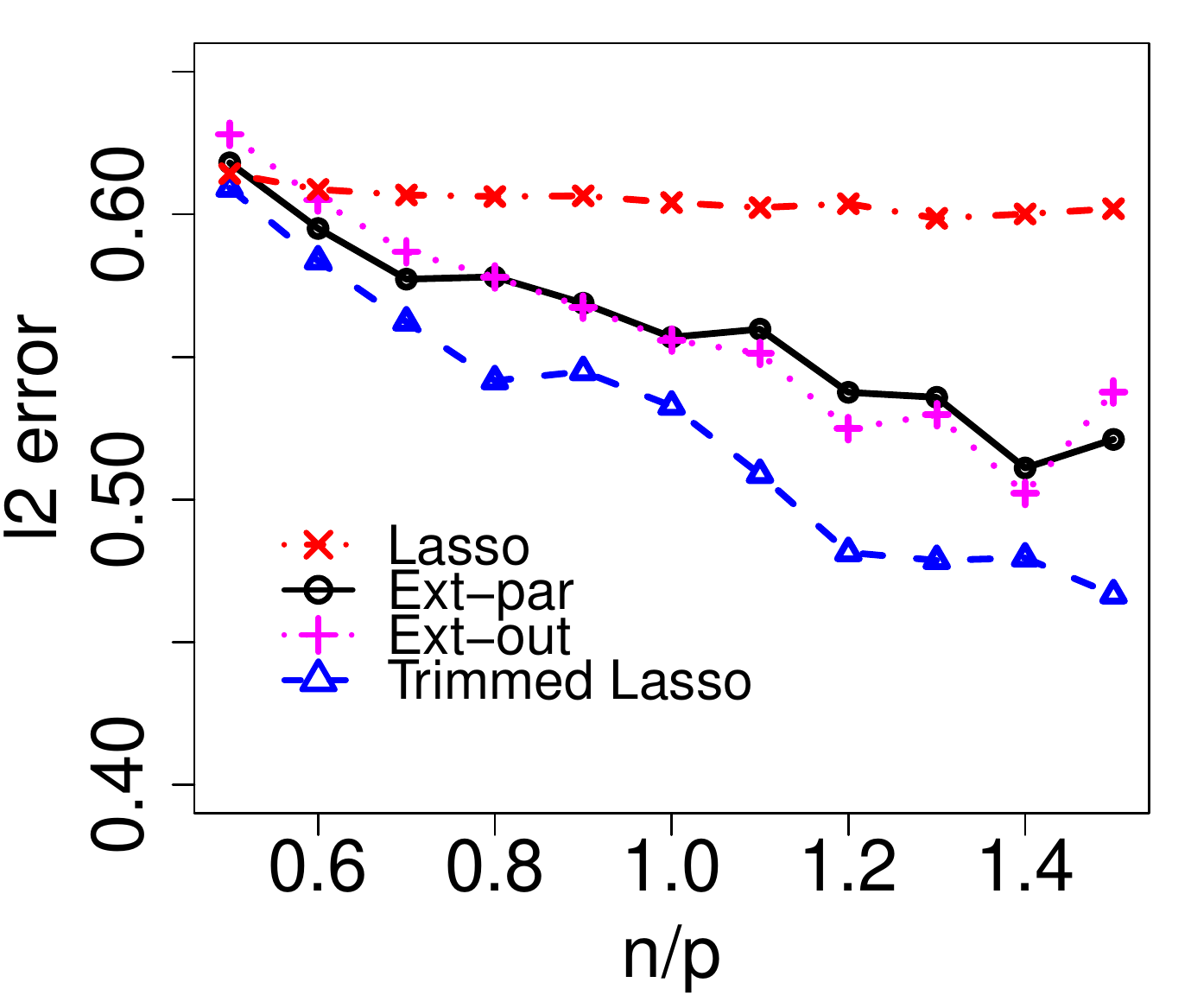}}
	\hspace{.5cm}
	\subfigure{\includegraphics[width=0.3\textwidth]{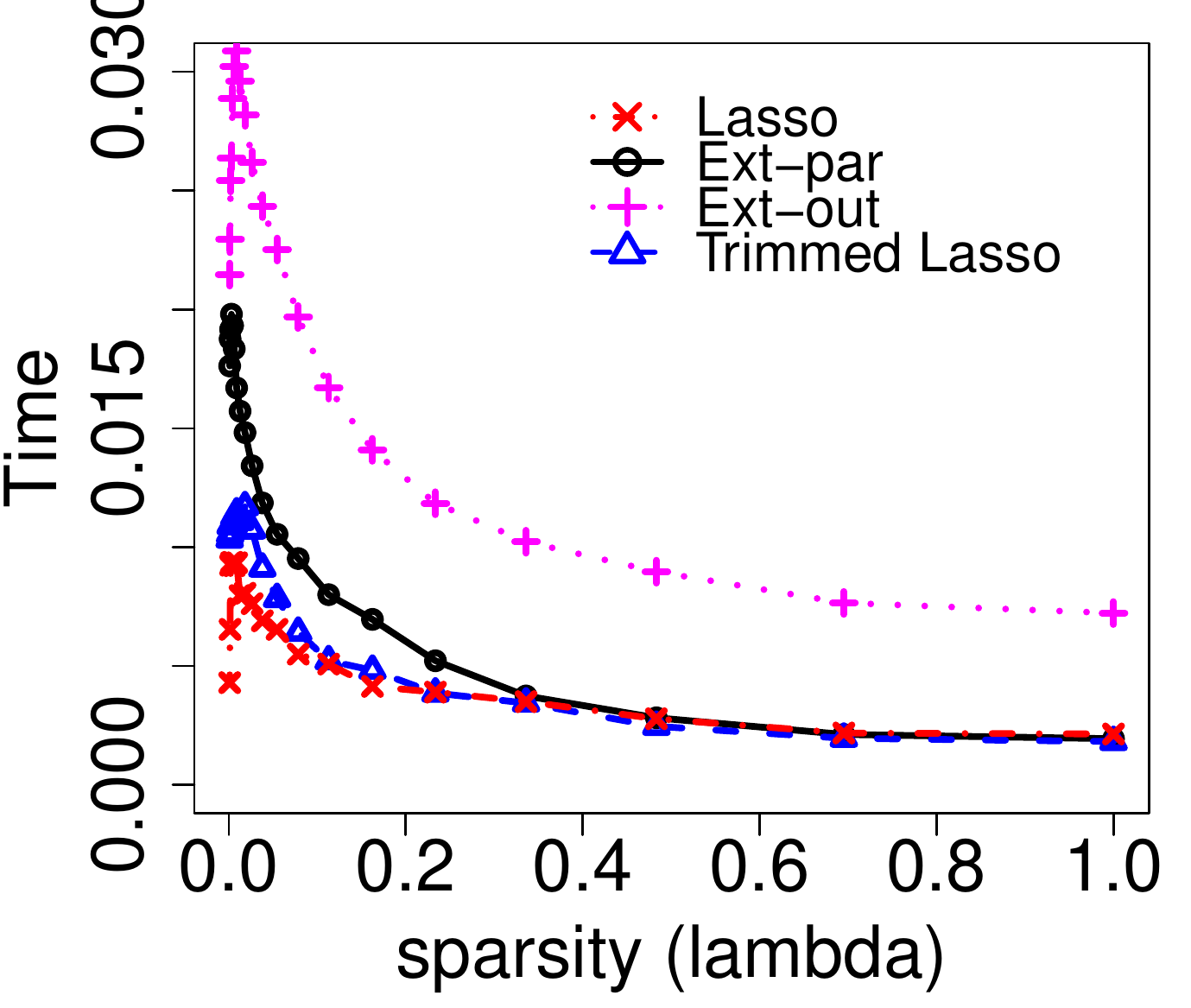}}
	\vspace{-.5cm}\caption{$\ell_2$ error vs.sample size $n$ under logistic regression model  (a) $\sqrt n$ corruptions (b) $0.1n$ corruptions. (c) Timing comparison for $0.1n$ corruptions  and $n=p.$}	
	\label{fig:logit}
\end{figure}

\subsection{Simulations for Trace-Norm Regularized Regression}
\begin{figure}[t]
	\centering
	\subfigure{\includegraphics[width=0.32\textwidth] {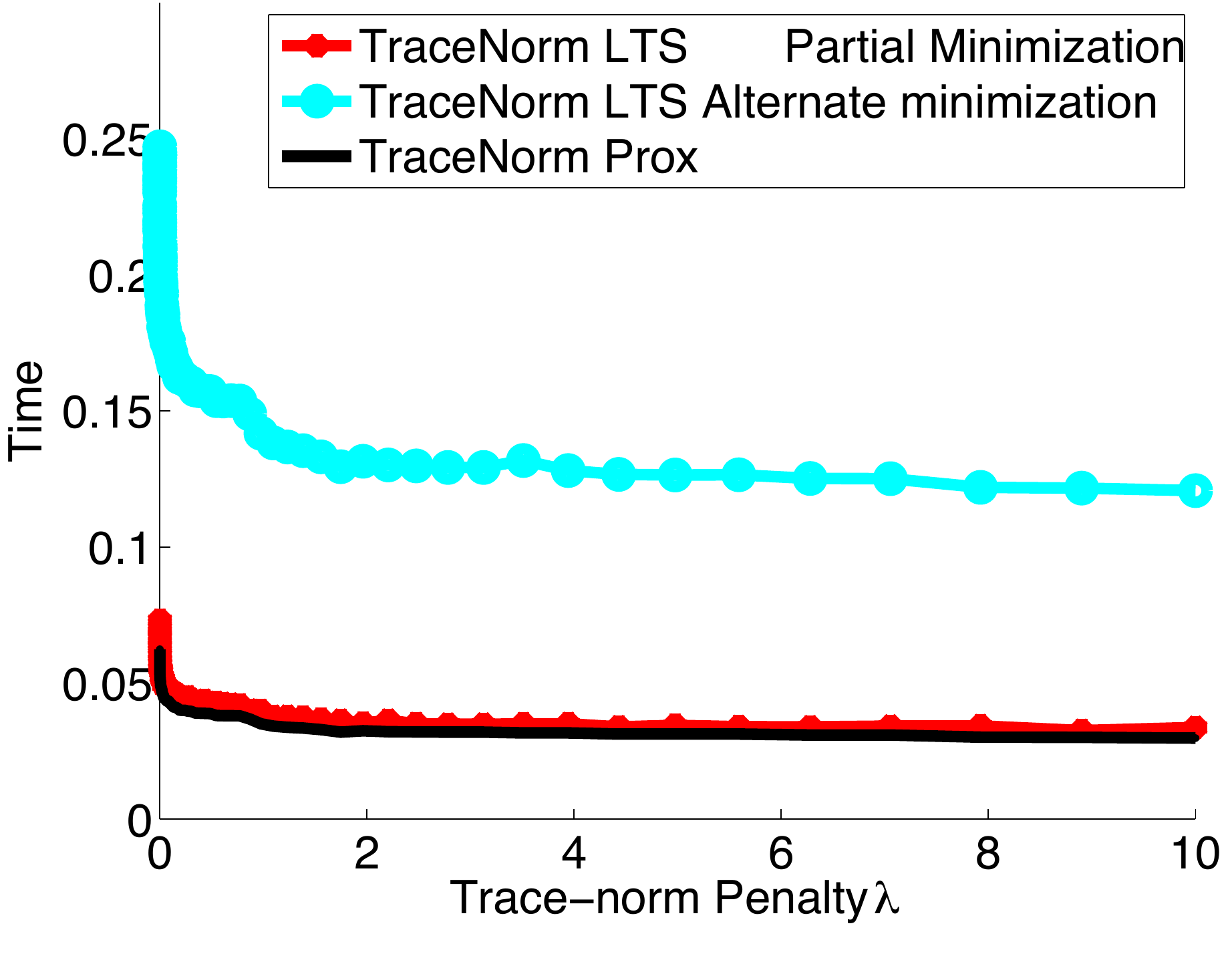}}
	\caption{Average timing of TraceNorm-LTS with partial minimization,TraceNorm-LTS with full alternate minimization, and TraceNorm-Prox under 20\% of contaminated data.}
	\label{fig:timing2}
\end{figure}

Beyond the $\ell_1$ penalty, we consider trace-norm regularized multi response regression.  We set $\R(\Theta)=\|\Theta\|_*$, for $\Theta\in\mathbb{R}^{p \times q}$. We consider $n=50$ samples, $p=300$ covariates, and $q=10$ responses. Each entry of $X$ is generated independently from $N(0,1).$ To generate the true low rank weights, we first sample a $p \times q$ matrix of coefficients, with each coefficient sampled independently from $N(0,1)$. We then set the true parameter matrix to the best rank 3 approximation of the sample, obtained using an SVD. For clean samples in $G$, we then set the error term as $\epsilon_i \sim   N(0,0.01).$ The contaminated terms are generated with an error term as  $\delta_i \sim N(2,1).$  We consider varying corruption levels ranging from $5\%$ to $30\%.$ The parameters are tuned as in the previous section and we present the average $\ell_2$ error based on 100 simulation runs. Figure~\ref{fig:timing2} further illustrates the computational advantage of the partial minimization scheme described in Section \ref{sec:opt} for general structures. 

\begin{table}[t!]
	\small
	\caption{Average $\ell_2$ error for comparison methods on simulated data under low-rank multi response linear models with contaminated data.} \label{ta::expresultsLR}
	\centering
\begin{tabular}{c|ccc}
\abovespace\belowspace
Contamination \% & No trimming &  Low-Rank LTS\\
\hline
5\% & 20.43 &  \textbf{19.20} \\
10\% & 33.49 & \textbf{25.10}\\
20\% & 33.70 &  \textbf{26.05} \\
30\% & 40.78& \textbf{30.10}\\
\hline
\end{tabular}
\end{table}


	\begin{figure}[t!]
		\addtolength{\subfigcapskip}{-0.08in}
		\centering
		\subfigure[M1]{\includegraphics[width=0.42\textwidth]{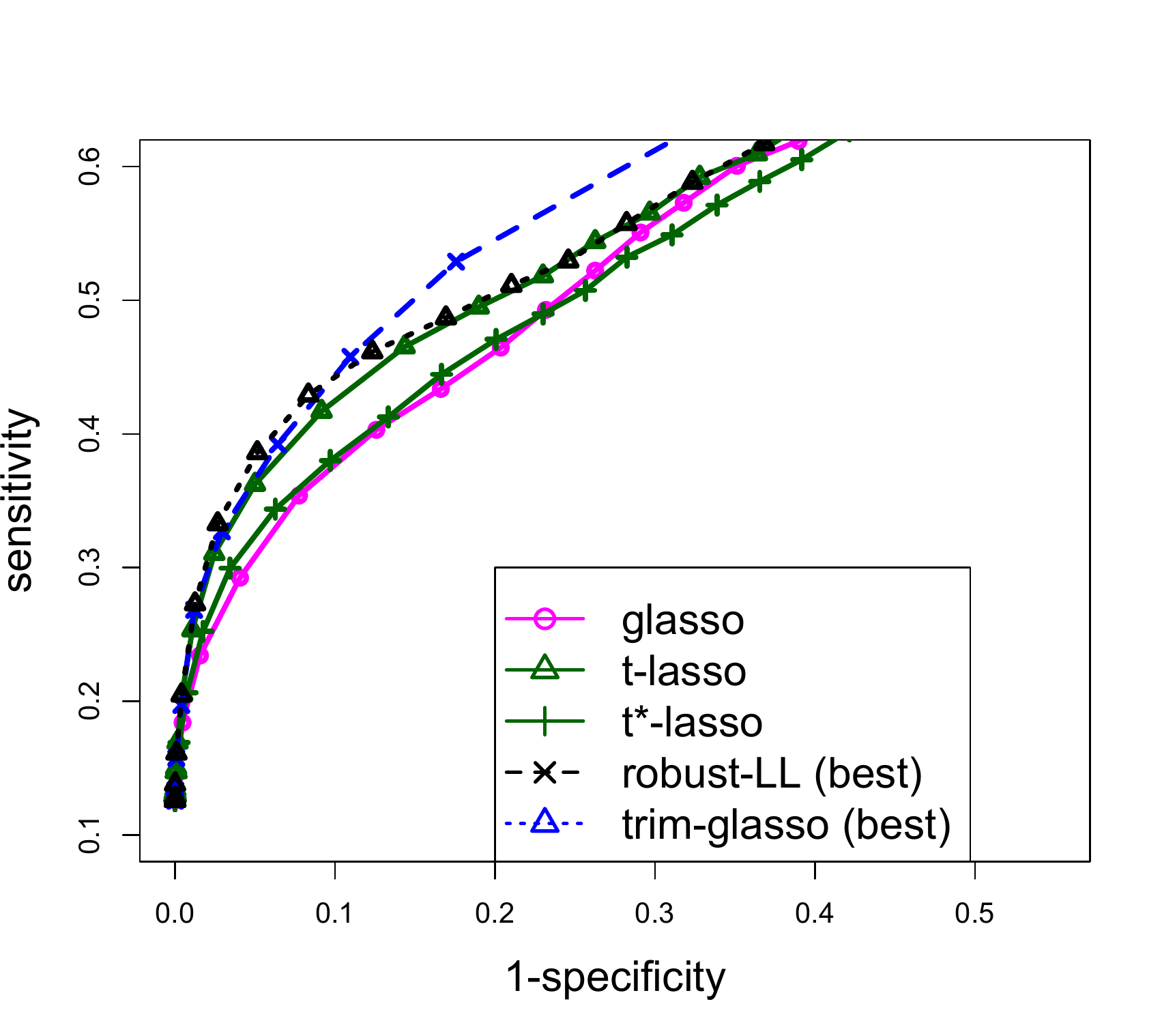}}
		\subfigure[M2]{\includegraphics[width=0.42\textwidth]{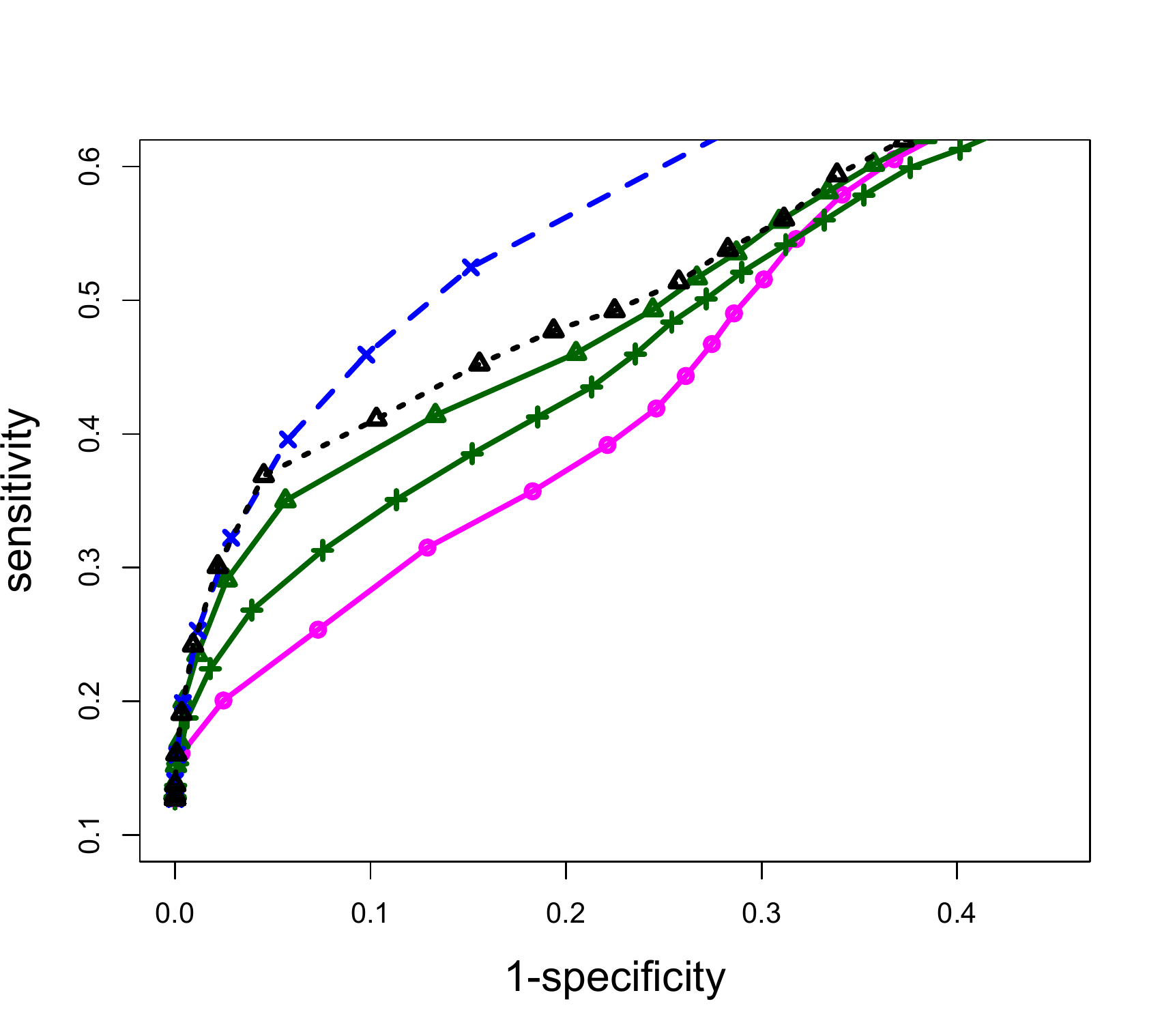}}\\
	 	\vspace{-.9cm}\subfigure[M3]{\includegraphics[width=0.42\textwidth]{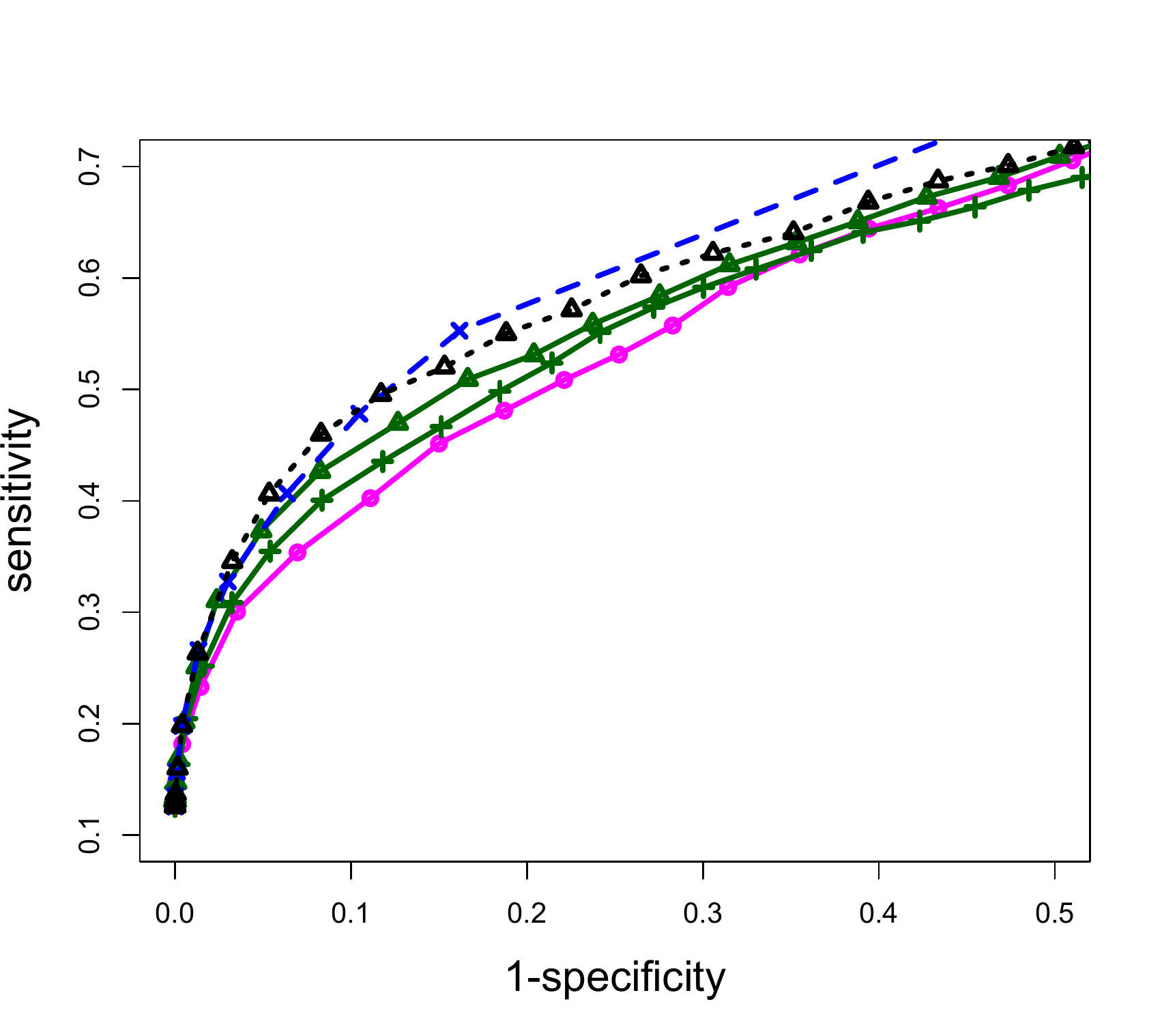}}
		\subfigure[M4]{\includegraphics[width=0.42\textwidth]{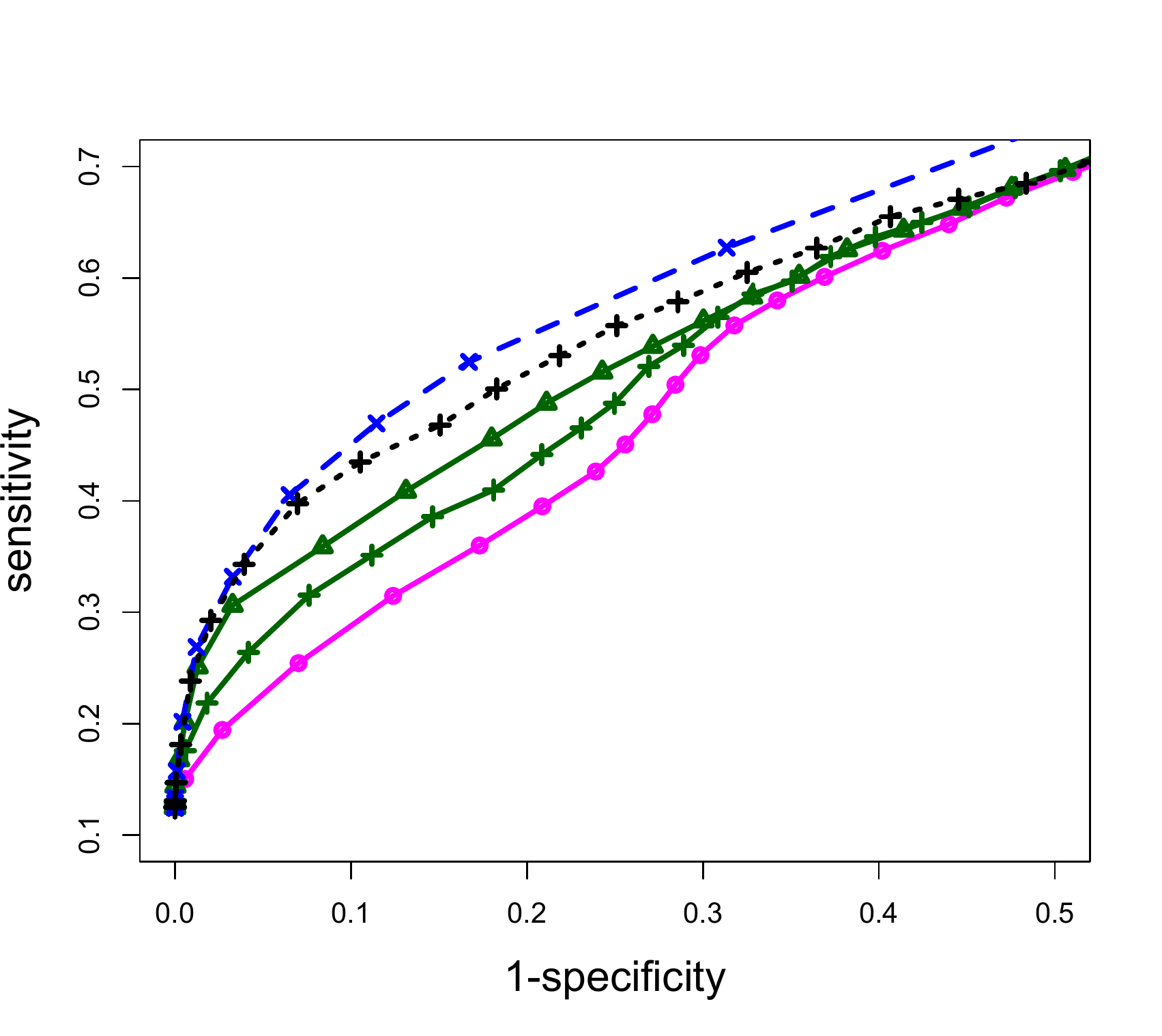}}
		\caption{Average ROC curves for the comparison methods for contamination scenarios M1-M4.}
		\label{fig:roc1}
	\end{figure}

\subsection{Simulations for Gaussian Graphical Models}

We compare the Trimmed Graphical Lasso ({trim-glasso}) algorithm against the vanilla Graphical Lasso({glasso})~\cite{FriedHasTib2007};  the  {t-lasso} and  {t*-lasso} methods~\cite{drton2011}, and {robust-LL}: the robustified-likelihood approach of~\cite{sun2012}.

Our simulation setup is similar to~\cite{sun2012} and is a akin to gene regulatory networks. Namely we consider four different scenarios where the outliers are generated from models with different graphical structures. Specifically, each sample is generated from the following mixture distribution:
\[
y_k \sim (1-p_0)N_p(0,\param^{-1}) + \frac{p_0}{2} N_p(-\mu, \param_o^{-1}) + \frac{p_0}{2} N_p(\mu, \param_o^{-1}),~~k=1,\ldots,n,
\]
where $p_o=0.1, n=100,$ and $p=150$. Four different outlier distributions are considered:
\begin{itemize}
	\item[] M1: $\mu=(1,\ldots,1)^T, \param_o=\tilde\param$, \quad  M2: $\mu=(1.5,\ldots,1.5)^T, \param_o=\tilde\param$,
	\item[] M3: $\mu=(1,\ldots,1)^T, \param_o=I_p$, \quad M4: $\mu=(1.5,\ldots,1.5)^T, \param_o=I_p$.
\end{itemize}

For each simulation run, $\param$ is a randomly generated precision matrix corresponding to a network with $9$ hub nodes simulated as follows. Let $A$ be the adjacency of the network. For all $i<j$ we set $A_{ij} = 1$ with probability 0.03, and zero otherwise. We set $A_{ji} = A_{ij}.$  We then randomly select 9 hub nodes and set the elements of the corresponding rows and columns of $A$ to one with probability 0.4 and zero otherwise.  Using $A$, the simulated nonzero coefficients of the precision matrix are sampled as follows. First we create a matrix $E$ so that $E_{i,j}=0$ if $A_{i,j}=0$, and $E_{i,j}$ is sampled uniformly from $[-0.75,-0.23] \cup [0.25,0.75]$ if $A_{i,j}\neq 0.$ Then we set $E=\frac{E + E^T}{2}.$ Finally we set $\param= E  +(0.1 - \Lambda_{\min}(E))I_p,$ where $\Lambda_{\min}(E)$ is the smallest eigenvalue of $E.$ 
$\tilde\param$  is a randomly generated precision matrix in the same way $\param$ is generated. 

For the  robustness parameter $\beta$ of the {robust-LL} method, we consider $\beta \in \{ 0.005, 0.01, 0.02,0.03\}$ as recommended in~\cite{sun2012}. For the {trim-glasso} method we consider  $\frac{100 h}{n}\in \{90, 85, 80\}.$ Since all the robust comparison methods converge to a stationary point, we tested various initialization strategies for the concentration matrix, including $I_p$, $(S + \lambda I_p)^{-1}$ and the estimate from {glasso}. We did not observe any noticeable impact on the results.

Figure~\ref{fig:roc1} presents the average ROC curves of the comparison methods over 100 simulation data sets for scenarios M1-M4 as the tuning parameter $\lambda$ varies. In the figure, for {robust-LL} and {trim-glasso} methods, we depict the best curves with respect to parameter $\beta$ and $h$ respectively.
The detailed results for all the values of $\beta$ and $h$ considered are provided in the appendix.

From the ROC curves we can see that our proposed approach is competitive compared the alternative robust approaches {t-lasso}, {t*-lasso} and {robust-LL.} The edge over {glasso} is even more pronounced for scenarios M2, M4. 
Surprisingly, {trim-glasso} with $h/n=80\%$ achieves superior sensitivity for nearly any specificity.


Computationally the {trim-glasso} method is also competitive compared to alternatives. The average run-time over the path of tuning parameters $\lambda$ is 45.78s for {t-lasso}, 22.14s for {t*-lasso}, 11.06s  for  {robust-LL}, 1.58s for {trimmed lasso}, 1.04s for {glasso}. Experiments were run on R in a single computing node with a Intel Core i5 2.5GHz CPU and 8G memory. For  {t-lasso},  {t*-lasso} and {robust-LL} we used the R implementations provided by the methods' authors. For {glasso} we used the {glassopath} package. 


\section{Application Genomic Analysis}\label{sec:real}
\begin{figure}[t]
		\centering
		\includegraphics[width=0.37\textwidth]{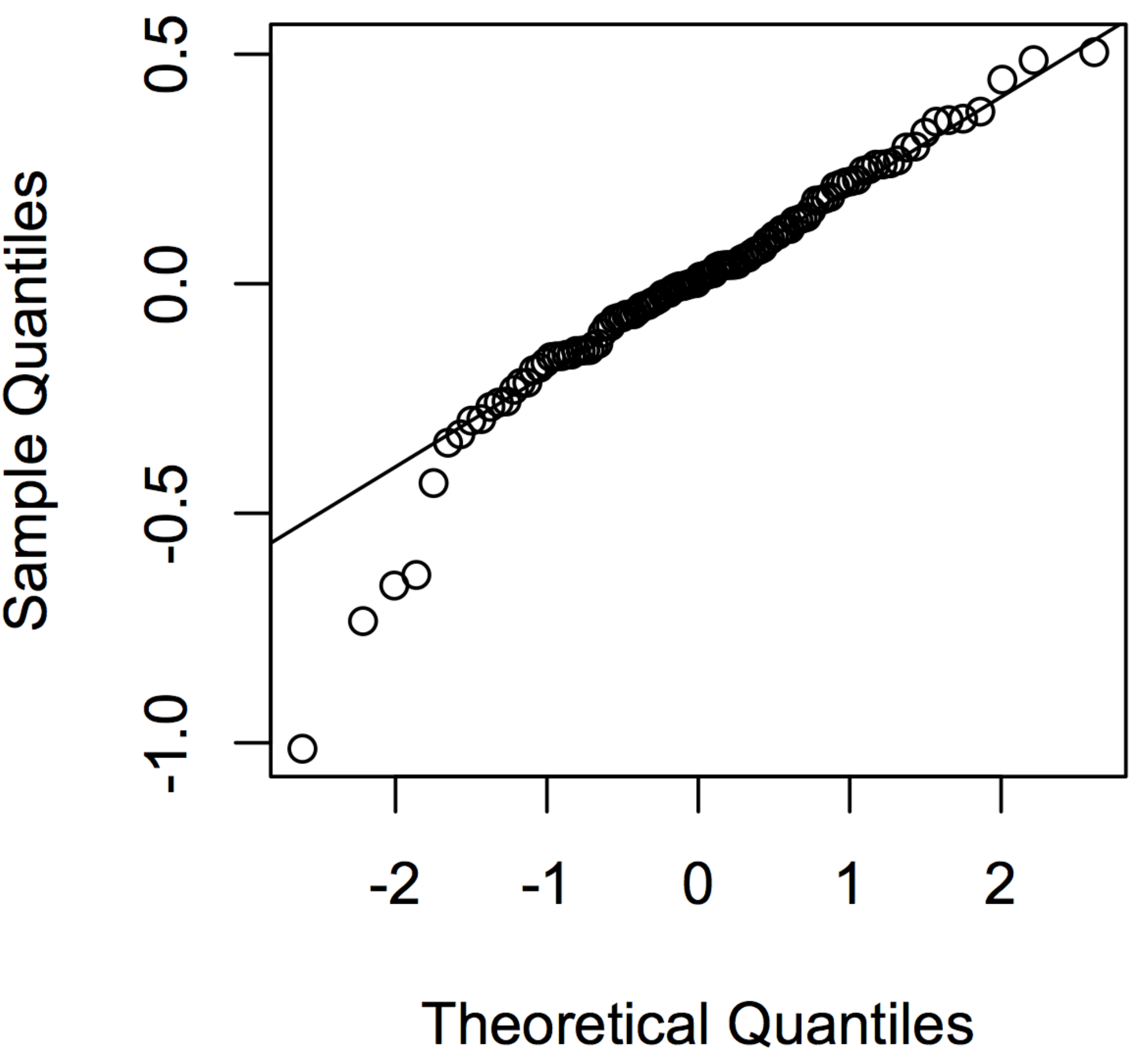}
		\caption{QQ-plots of fitted residuals for the Sparse-LTS method in the genomic study.}
		\label{fig:qqplot}
\end{figure}

\begin{table}[t]
\small
	\caption{Average Trimmed Mean Square Error from 10-fold cross validation for comparison methods on the Yeast dataset.} \label{ta::yeastmse}
	\begin{center}
		\begin{tabular}{l| c}
		\abovespace\belowspace
		Method & T-MSE\\
		\hline
		Lasso &  0.137\\
		LAD-Lasso & 0.132\\
		Extended Lasso &  0.093\\
		ROMP & 0.135\\
		{\bfseries Sparse-LTS} &   {\bfseries 0.081}\\
		\hline
	\end{tabular}
	\end{center}
		\caption{Marker position of SNPs selected on chromosome 8 by comparison methods for the Yeast dataset.} \label{ta::yeastsel}
		\begin{center}
		\begin{tabular}{c | c }
		\abovespace\belowspace
		 {\small LAD-Lasso} &   {\small Sparse-LTS}\\
		\hline
		111682 &   46007\\
		 213237 &	  46055\\
		  &111682\\
		 &  111683\\
		 & 111686\\
		 &111687\\
		&  111690\\
		\hline
		\end{tabular}
	\end{center}
\end{table}

\subsection{Analysis of Yeast Genotype and Expression data}

We apply Sparse-LTS 
, Extended Lasso~\citep{Nguyen2011b}, LAD Lasso~\citep{Wang2007},  standard Least Squares Lasso estimator~\citep{Tibshirani96}, and ROMP~\citep{Chen13}  to the analysis of yeast genotype and gene expression data. We employ the ``yeast'' dataset  from~\cite{Brem}. The data set concerns $n=112$ \emph{F1 segregants} from a yeast genetic cross between two strains:  \emph{BY} and \emph{RM}.  For each of these 112 samples, we observe $p=3244$ SNPs (These genotype data are our predictors $x$)  and focus on the gene expression of gene \emph{GPA1} (our response $y$), which is involved in pheromone response~\cite{Brem}. For both Sparse-LTS-Ada and Sparse LTS considering a total of $|B|=11$ contaminated observations lead to the best predictive performance on the uncontaminated data. In addition, the QQ-plots of the fitted residuals from the various comparison methods indicated heavy left tails (see Figure~\ref{fig:qqplot}). This suggests that it might be advisable to use robust methods. 

We compare the trimmed mean square error (T-MSE) computed from 10-folds cross validation for each method, where for each method we exclude the 11 observations with largest residual absolute error.  From Table~\ref{ta::yeastmse} we can see that
Sparse-LTS exhibit the smallest T-MSE. 



We conclude by examining the SNPs selected by the  methods achieving the lowest T-MSE: 
Sparse-LTS and LAD Lasso.  Out of $p=3244$ SNPs,  Sparse-LTS selected 30 SNPs, 
and  LAD Lasso chose 61 SNPs. Table~\ref{ta::yeastsel} provides a list of the SNPs selected on chromosome 8, which is where gene \emph{GPA1} resides. In the dataset, there is a total of 166 SNPs on chromosome 8. From the table we can see that there is some overlap in terms of the selected SNPs across the various methods. Sparse-LTS tends to select a larger number of SNPs on chromosome 8 even though it selects fewer SNPs in total (namely within and beyond chromosome 8). Five of these are very close to \emph{GPA1} which is consistent with the fact that \emph{GPA1} can directly inhibit the mating signal by binding to its own subunit~\cite{mating}.

\subsection{Application to the analysis of Yeast Gene Expression Data}

We analyze a yeast microarray dataset generated by~\cite{brem2005landscape}. The dataset concerns $n=112$ yeast segregants (instances). We focused on $p=126$ genes (variables) belonging to cell-cycle pathway as provided by the KEGG database~\cite{kegg}.  For each of these genes we standardize the gene expression data to zero-mean and unit standard deviation. We observed that the expression levels of some genes are clearly not symmetric about their means and might include outliers. For example
the histogram of gene \emph{ORC3} is presented in Figure~\ref{fig:hist}(a). 
\begin{figure}[t]
		\centering
		\begin{tabular}{cc}

\includegraphics[width=0.45\textwidth]{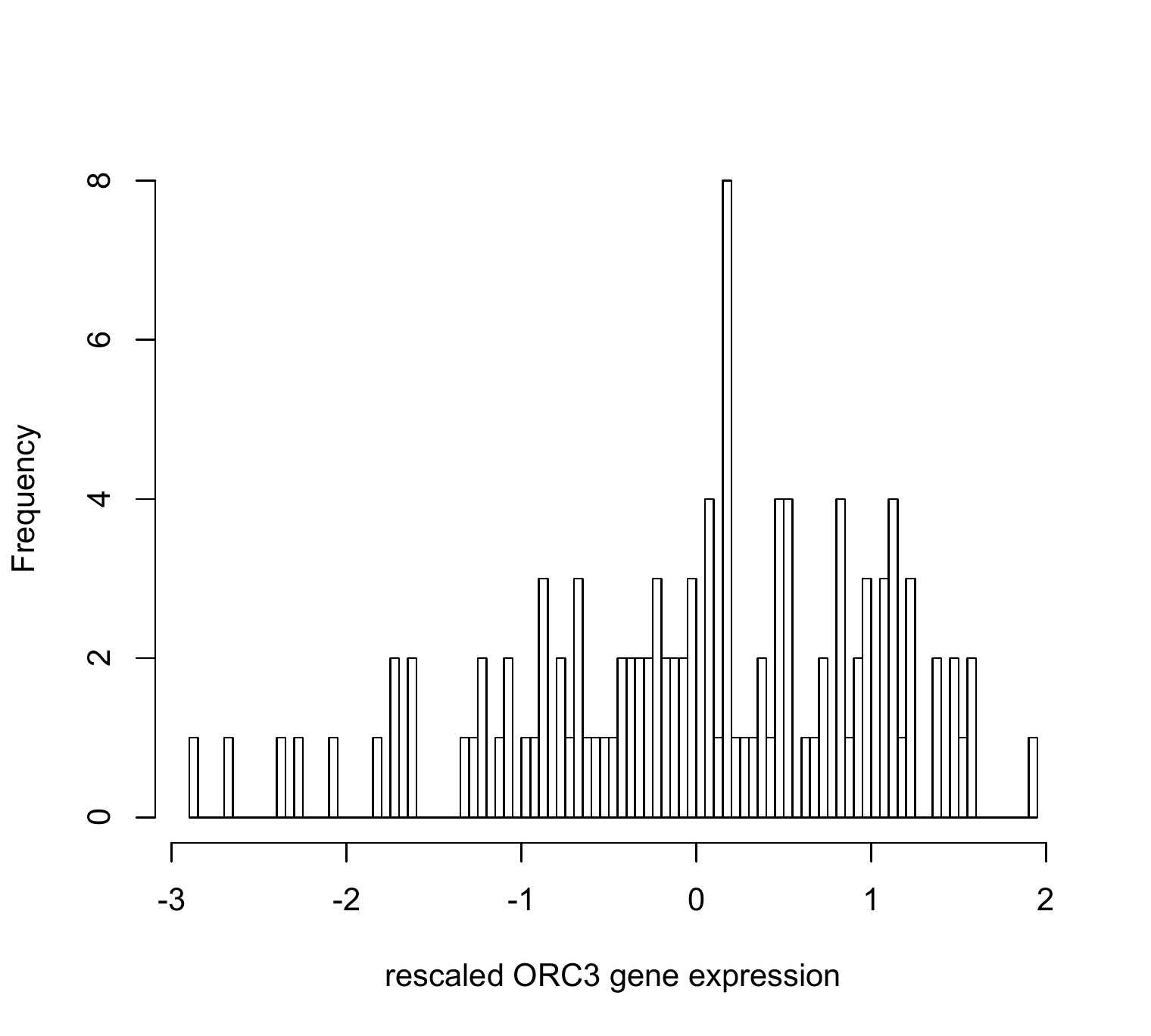} & \includegraphics[width=0.45\textwidth]{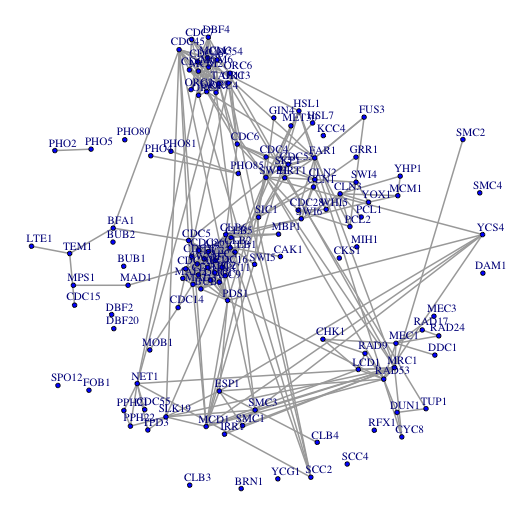}
\end{tabular}
		\caption{(a) Histogram of standardized gene expression levels for gene \emph{ORC3}.  (b) Network estimated by {trim-glasso}.}
		\label{fig:hist}
	\end{figure}
	For the {robust-LL} method we set $\beta=0.05$  and for {trim-glasso} we use $h/n=80\%.$
We use 5-fold-CV to choose the tuning parameters for each method. After $\lambda$ is chosen for each method, we rerun the methods using the full dataset to obtain the final precision matrix estimates. 

Figure~\ref{fig:hist}(b) shows the cell-cycle pathway estimated by our proposed  method. For comparison the  cell-cycle pathway from the KEGG~\cite{kegg} is provided in Figure~\ref{fig:cell}.
\begin{figure}[!ht]
		\centering		
		\includegraphics[width=\textwidth]{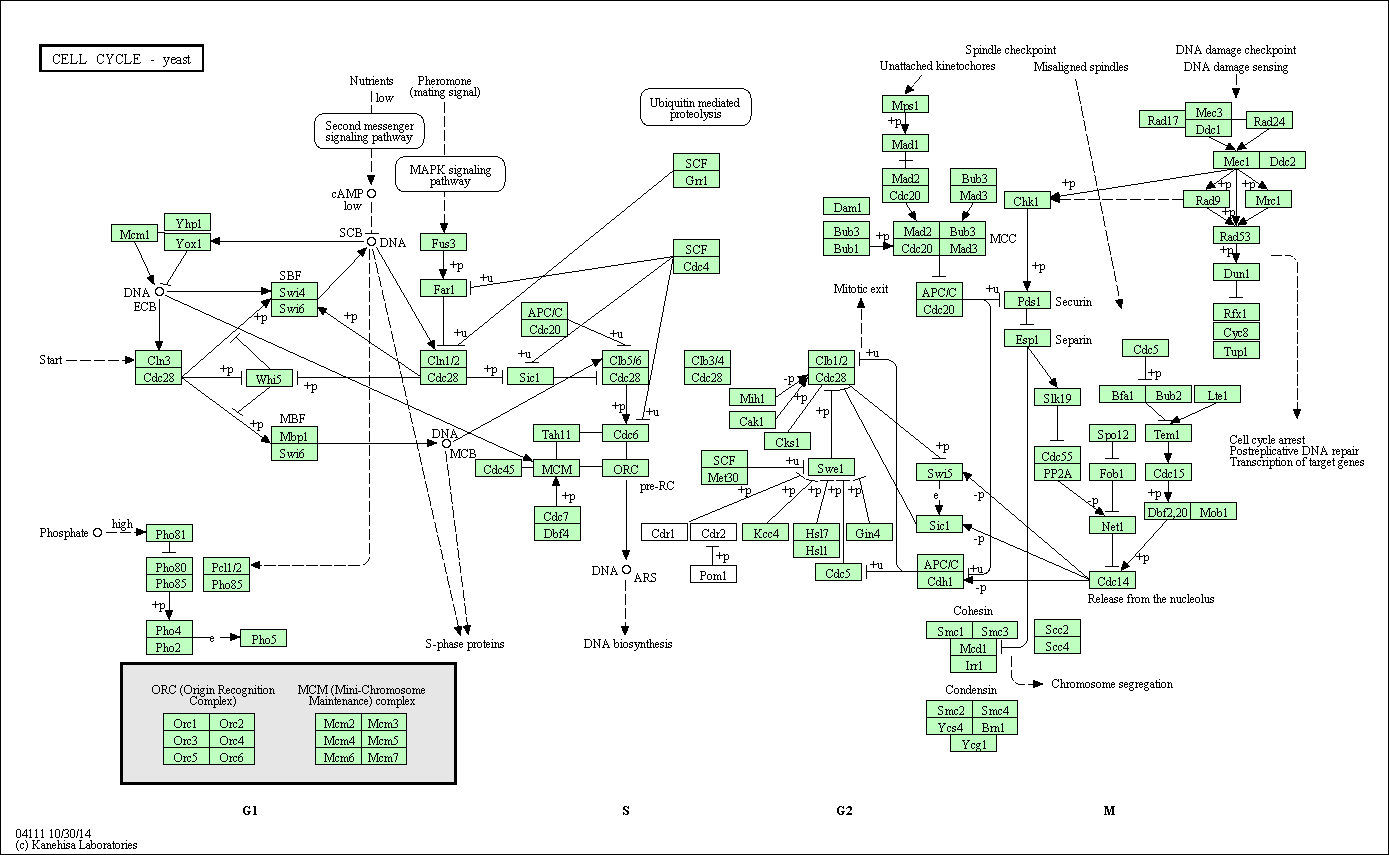}
		\caption{Reference Yeast Cell Signaling Network from the KEGG database~\citep{kegg}.}
		\label{fig:cell}
\end{figure}
It is important to note that the KEGG graph corresponds to what is currently known about the pathway. It should \emph{not} be treated as the ground truth. Certain discrepancies between KEGG and estimated graphs may also be caused by inherent limitations in the dataset used for modeling. For instance, some edges in cell-cycle pathway may not be observable from gene expression data. Additionally, the perturbation of cellular systems might not be strong enough to enable accurate inference of some of the links. 

{glasso} tends to estimate more links than the robust methods. We postulate that the lack of robustness might result in inaccurate network reconstruction and the identification of spurious links.  Robust methods tend to estimate networks that are more consistent with that from the KEGG ($F_1$-score of 0.23 for {glasso}, 0.37 for {t*-lasso}, 0.39 for {robust-NLL} and 0.41 for {trim-glasso}, where the $F_1$ score is the harmonic mean between precision and recall). For instance our approach recovers several characteristics of the KEGG pathway. For instance, genes \emph{CDC6}  (a key regulator of DNA replication playing important roles in the activation and maintenance of the checkpoint mechanisms coordinating S phase and mitosis) and \emph{PDS1} (essential gene for meiotic progression and mitotic cell cycle arrest) are identified as a hub genes, while genes \emph{CLB3,BRN1,YCG1} are unconnected to any other genes.

\section{Concluding Remarks}
We presented a family of trimmed estimators for a wide class of structured high-dimensional problems. We provided general results on their statistical convergence rates and consistency.  In particular our results for sparse linear regression and gaussian graphical models allow to precisely characterize the impact of corruptions on the statistical performance of the resulting estimatiors, while recovering the rates of their `untrimmed' counterparts under clean data. 
We showed how to efficiently adapt existing optimization algorithms to solve the modified trimmed problems. 
Relevant directions for future work include specializing our theoretical analysis to generalized linear models, applying and analyzing trimmed approaches for more general structural regularizations, and the study of concomittent selection of the amount of trimming.

\newpage

{\small
\bibliographystyle{asa}
\bibliography{sml,LTS,Elem_GLM,robggm}
}

\newpage
\appendix
\section*{Appendix}

%
%

\section{Proof of Theorem \ref{ThmGeneral}}
We use the shorthand for local optimal error vector: $\Lerrt := \Ltheta - \Ttheta$ and $\Lerrw := \Lw - \Tw$ where $(\Ltheta,\Lw)$ is an \emph{arbitrary} local optimum of $M$-estimator of \eqref{EqnRobustGeneral}. Our proof mainly uses the fact that $(\Ltheta,\Lw)$ is a local minimum of \eqref{EqnRobustGeneral} satisfying
\begin{align*}
	\biginner{\nabla_{\theta} \Loss\big(\Ttheta + \Lerrt,\Tw+\Lerrw \big)}{\Ltheta - \theta}  \leq - \biginner{\partial \lambda\|\Ttheta + \Lerrt\|_1}{\Ltheta - \theta} \, \quad \text{for any feasible } \theta.
\end{align*}
This inequality comes from the first order stationary condition (see \citet{LW15} for details) in terms of only $\theta$ fixing $w$ at $\Lw$. In order to provide the complete proof of the theorem, we need to define the set of notations on the model space, perturbation space and  corresponding projections following \citet{NRWY12}. The sparse LTS \eqref{EqnRobustLS} is a typical example of \eqref{EqnRobustGeneral}, and such notations can be naturally defined based on the true support set $S$. In this proof, we specifically focus on the case with $\R(\cdot) := \|\cdot\|_1$ for notational simplicity, but statements here can be seamlessly extendible for the general regularizer $\R(\cdot)$ and the appropriately defined model/perturbation spaces. 

If we take $\theta = \Ttheta$ above, we have
\begin{align}\label{EqnLocalMinima}
	& \biginner{\nabla_{\theta} \Loss\big(\Ttheta + \Lerrt,\Tw+\Lerrw \big)}{\Lerrt} \leq - \biginner{\partial \lambda\|\Ttheta + \Lerrt\|_1}{\Lerrt} \overset{(i)}{\leq}  \lambda( \|\Ttheta\|_1 - \|\Ltheta\|_1 )\nonumber\\\
	\leq \, &  \lambda( \|\Ttheta\|_1 + \|\Lerrt_{S^c}\|_1  - \|\Lerrt_{S^c}\|_1  - \|\Ltheta\|_1 ) = \lambda( \|\Ttheta + \Lerrt_{S^c}\|_1  - \|\Lerrt_{S^c}\|_1  - \|\Ltheta\|_1 )  \nonumber\\
	\overset{(ii)}{\leq} \, & \lambda\big( \|\Ttheta + \Lerrt_{S^c} + \Lerrt_{S}\|_1 + \|\Lerrt_{S}\|_1 - \|\Lerrt_{S^c}\|_1  - \|\Ltheta\|_1 \big) = \lambda(\|\Lerrt_{S}\|_1 - \|\Lerrt_{S^c}\|_1 ) \, ,
\end{align}
where $S$ is true support set of $\theta^*$, the inequalities $(i)$ and $(ii)$ hold by respectively the convexity and the triangular inequality of $\ell_1$ norm. 

Now, by the RSC condition in \ref{Con:rsc}, we obtain
\begin{align*}
	& \RSCcon \|\Lerrt\|_2^2  - \RSCtolOne \|\Lerrt\|_1^2 \nonumber\\
	\leq \, & \biginner{\nabla_{\theta} \Loss\big(\Ttheta + \Lerrt,\Tw\big) - \nabla_{\theta}\Loss\big(\Ttheta,\Tw\big) }{\Lerrt} \nonumber\\
	= \, & \Biginner{\nabla_{\theta} \Loss\big(\Ttheta + \Lerrt,\Tw + \Lerrw \big) - \nabla_{\theta} \Loss\big(\Ttheta + \Lerrt,\Tw + \Lerrw \big) + \nabla_{\theta} \Loss\big(\Ttheta + \Lerrt,\Tw\big) - \nabla_{\theta}\Loss\big(\Ttheta,\Tw\big) }{\Lerrt}\, ,
\end{align*}
which is equivalent with 
\begin{align}\label{EqnMainTmp1}
	& \RSCcon \|\Lerrt\|_2^2  - \RSCtolOne \|\Lerrt\|_1^2 + \biginner{\nabla_{\theta} \Loss\big(\Ttheta + \Lerrt,\Tw + \Lerrw \big) - \nabla_{\theta} \Loss\big(\Ttheta + \Lerrt,\Tw\big)}{\Lerrt} \nonumber\\
	\leq \, & \biginner{\nabla_{\theta} \Loss\big(\Ttheta + \Lerrt,\Tw + \Lerrw \big) - \nabla_{\theta}\Loss\big(\Ttheta,\Tw\big) }{\Lerrt}.
\end{align}
Combining \eqref{EqnLocalMinima}, \eqref{EqnMainTmp1} and \ref{Con:inc} yields 
\begin{align*}
	& \RSCcon \|\Lerrt\|_2^2  - \RSCtolOne \|\Lerrt\|_1^2 - \IncConOne \|\Lerrt\|_2 - \IncConTwo \|\Lerrt\|_1 \nonumber\\ 
	\leq \, & - \biginner{\nabla_{\theta}\Loss\big(\Ttheta,\Tw\big) }{\Lerrt} + \lambda\, \big(\|\Lerrt_{S}\|_1 - \|\Lerrt_{S^c}\|_1 \big) \nonumber\\
	\leq \, & \big\| \nabla_{\theta}\Loss\big(\Ttheta,\Tw\big) \big\|_\infty \|\Lerrt\|_1 + \lambda\, \big(\|\Lerrt_{S}\|_1 - \|\Lerrt_{S^c}\|_1 \big) \, .
\end{align*}
Since the theorem assumes $\max \big\{ \| \nabla_{\theta}\Loss\big(\Ttheta,\Tw\big) \|_\infty, 2\rho\RSCtolOne + \IncConTwo  \big\} \leq \frac{\lambda}{4}$, we can conclude that 
\begin{align}\label{EqnMainTmp2}
	& 0 \leq \RSCcon \|\Lerrt\|_2^2  \nonumber\\
	\leq \, &  \big\| \nabla_{\theta}\Loss\big(\Ttheta,\Tw\big) \big\|_\infty \|\Lerrt\|_1 + \lambda\, \big(\|\Lerrt_{S}\|_1 - \|\Lerrt_{S^c}\|_1 \big) \nonumber\\
		& \quad + \Big( 2 \rho \RSCtolOne + \IncConTwo \Big) \|\Lerrt\|_1 + \IncConOne \|\Lerrt\|_2 \nonumber\\
	\leq \, & \frac{3\lambda}{2}\|\Lerrt_{S}\|_1 - \frac{\lambda}{2}\|\Lerrt_{S^c}\|_1  + \IncConOne \|\Lerrt\|_2 .
\end{align}
As a result, we can finally have an $\ell_2$ error bound as follows:
\begin{align*}
	& \RSCcon \|\Lerrt\|_2^2  \leq \frac{3\lambda}{2}\|\Lerrt_{S}\|_1  + \IncConOne \|\Lerrt\|_2 \nonumber\\
	\leq \, & \frac{3\lambda\sqrt{k}}{2}\|\Lerrt_{S}\|_2  + \IncConOne \|\Lerrt\|_2 \leq \Big( \frac{3\lambda\sqrt{k}}{2}  + \IncConOne \Big) \|\Lerrt\|_2 
\end{align*}
implying that 
\begin{align*}
	\|\Lerrt\|_2 \leq \frac{1}{\RSCcon}\Big( \frac{3\lambda\sqrt{k}}{2}  + \IncConOne \Big)  \, .
\end{align*}
At the same time in order to derive $\ell_1$ error bound, we again use the inequality by \eqref{EqnMainTmp2}:
\begin{align*}
	\|\Lerrt_{S^c}\|_1 \leq 3 \|\Lerrt_{S}\|_1  + \frac{2}{\lambda} \IncConOne \|\Lerrt\|_2 \,.
\end{align*}
Hence, 
\begin{align*}
	\|\Lerrt\|_1 \leq \, & \|\Lerrt_{S}\|_1 + \|\Lerrt_{S^c}\|_1 \leq 4 \|\Lerrt_{S}\|_1  + \frac{2}{\lambda} \IncConOne \|\Lerrt\|_2 \leq 4\sqrt{k} \|\Lerrt_{S}\|_2  + \frac{2}{\lambda} \IncConOne \|\Lerrt\|_2 \nonumber\\
	\leq \, & \Big( 4\sqrt{k} + \frac{2}{\lambda} \IncConOne\Big) \|\Lerrt\|_2 \leq \frac{2}{\lambda\, \RSCcon}\Big( 2\lambda \sqrt{k} + \IncConOne\Big)^2 \, ,
\end{align*}
which completes the proof.

\section{Proof of Corollary \ref{Cor:LS} and Corollary \ref{Cor:XCor} (Results for LTS)}

We begin with specifying \ref{Con:rsc} and \ref{Con:inc} for the showcasing example of \eqref{EqnRobustLS}:
\begin{align}
	& \frac{1}{h}\sum_{i=1}^n \Tw_i \inner{x_i}{\Lerrt}^2 \geq \RSCcon \|\Lerrt\|_2^2 - \RSCtolOne \R(\Lerrt)^2\, , \quad \text{and} \label{EqnLSrsc}\\
	& \frac{1}{h}\sum_{i=1}^n \Lerrw_i \big(\inner{x_i}{\Ttheta+\Lerrt} - y_i \big) \inner{x_i}{\Lerrt} \geq  - \IncConOne \|\Lerrt\|_2 - \IncConTwo \R(\Lerrt) \, . \label{EqnLSInc}
\end{align}
 
In order to directly utilize Theorem \ref{ThmGeneral} for linear models, we only need to show that \eqref{EqnLSrsc} (for the condition \ref{Con:rsc}) and \eqref{EqnLSInc} (for \ref{Con:inc}) hold. Throughout the proof, we use the fact that all elements in $\Lerrw$ corresponding to $G$ (set of good examples) are all zeros: $\Lerrw_G = {\bf 0}$ by construction.

First, consider the condition \ref{Con:rsc} in \eqref{EqnLSrsc}: $\frac{1}{h}\sum_{i=1}^n \Tw_i \inner{x_i}{\errt}^2$. Recall that we constructed $\Tw$ as follows: $\Tw_i$ is simply set to $\Lw_i$ if $i \in G$, and $\Tw_i =0$ for $i \in B$. Hence, by construction, $\sum_{i \in G} \Tw_i \geq h - (n-h)$ (since ${\bf 1}^\top w = h$), and at least $\frac{h-(n-h)}{2}$ samples in $G$ have $\Lw_i$ (therefore $\Tw_i$) larger than $\frac{h-(n-h)}{2h}$. Let $\widebar{G}$, which is the subset of $G$, be the set of such samples.

Then, $\frac{1}{h}\sum_{i=1}^n \Tw_i \inner{x_i}{\errt}^2$ can be lower bounded as follows:
\begin{align*}
	\frac{1}{h}\sum_{i=1}^n \Tw_i \inner{x_i}{\errt}^2 = \frac{1}{h}\sum_{i\in G} \Tw_i \inner{x_i}{\errt}^2 \geq \frac{1}{h}\sum_{i\in \bar{G}} \Tw_i \inner{x_i}{\errt}^2 \geq \frac{h-(n-h)}{2h^2}\sum_{i\in \bar{G}} \inner{x_i}{\errt}^2 \,. 
\end{align*}
Noting that all $x_i \in \widebar{G}$ are uncorrupted and iid sampled from $N(0,\SigG)$, we can appeal to the result in \citet{Raskutti2010}: with probability at least $1-c_1\exp \big(-c_2 |\widebar{G}|\big)$,
\begin{align}
	\frac{1}{|\widebar{G}|}\sum_{i\in \bar{G}} \inner{x_i}{\errt}^2 \geq \kappa_1 \|\errt\|_2^2  - \kappa_2 \frac{\log p}{|\widebar{G}|} \|\errt\|_1^2 \quad \text{for all } \errt \in \reals^p
\end{align}
where $\kappa_1$ and $\kappa_2$ are strictly positive constants depending only on $\SigG$. Therefore, 
\begin{align*}
	\frac{1}{h}\sum_{i=1}^n \Tw_i \inner{x_i}{\errt}^2 \geq  \frac{\big(h-(n-h)\big) |\widebar{G}|}{2h^2} \bigg(\kappa_1 \|\errt\|_2^2  - \kappa_2 \frac{\log p}{|\widebar{G}|} \|\errt\|_1^2  \bigg) \, , 
\end{align*}
hence, \eqref{EqnLSrsc} holds with 
\begin{align}\label{EqnLSrscConst}
	& \RSCcon = \frac{\kappa_1(2h-n)^2 }{4h^2}  \, , \, \RSCtolOne = \frac{\kappa_2(2h-n)\log p}{2h^2} 
\end{align}
since $|\widebar{G}| \geq h-(n-h)$ as discussed.

Now, we consider the condition \ref{Con:inc} in \eqref{EqnLSInc}.
\begin{align*}
	&\frac{1}{h}\sum_{i=1}^n \errw_i \big(\inner{x_i}{\Ttheta+\errt} - y_i \big) \inner{x_i}{\errt} = \frac{1}{h}\sum_{i=1}^n \errw_i \inner{x_i}{\errt}^2 - \frac{1}{h}\sum_{i=1}^n \errw_i \big(\epsilon_i + \delta_i\big) \inner{x_i}{\errt} \nonumber\\
	\geq \, & - \frac{1}{h}\sum_{i=1}^n \errw_i \big(\epsilon_i + \delta_i\big) \inner{x_i}{\errt} 
\end{align*}
where the inequality comes from \eqref{EqnLinearModels} and from the fact that $\errw_i$ is always greater than $0$: if $i \in G$, $\errw_i = 0$, and if $i \in B$, $\errw_i := \Lw_i - \Tw_i \geq 0$ since $\Lw_i \geq 0$ and $\Tw_i =0$.   

Now, we follow similar strategy as in \citet{Nguyen2011b}: given $\errt$, we divide the index of $\errt$ into the disjoint exhaustive subsets $S_1, S_2, \hdots, S_q$ of size $|B|$ such that $S_1$ contains $|B|$ largest absolute elements in $\errt$, and so on. Then, we have
\begin{align*}
	&\Big|\sum_{i=1}^n \errw_i \big(\epsilon_i + \delta_i\big) \inner{x_i}{\errt}\Big| = \Big|\sum_{i\in B} \errw_i \delta_i \inner{x_i}{\errt}\Big| = \Big|\sum_{i\in B} \errw_i \delta_i \sum_{j=1}^q\inner{[x_i]_{S_j}}{[\errt]_{S_j}}\Big| \nonumber\\
	\leq \, &  \sum_{j} \Big|\sum_{i\in B} \errw_i \delta_i \inner{[x_i]_{S_j}}{[\errt]_{S_j}}\Big| \leq \sum_{j} \sqrt{{\textstyle \sum_{i\in B}} \errw_i^2 \delta_i^2 } \sqrt{ {\textstyle \sum_{i\in B}} \inner{[x_i]_{S_j}}{[\errt]_{S_j}}^2 } \nonumber\\
	\leq \, & \sqrt{{\textstyle \sum_{i\in B}} \errw_i^2 \delta_i^2 } \, \Big(\max_j \matNorm{X^B_{S_j}}_2\Big) \, \sum_{j} \|[\errt]_{S_j} \|_2 \leq \sqrt{|B|} \Big(\max_{i \in B} \big| \errw_i \delta_i\big|\Big)\, \Big(\max_j \matNorm{X^B_{S_j}}_2\Big) \, \sum_{j} \|[\errt]_{S_j} \|_2 \nonumber\\
	\leq \, & \sqrt{|B|} \max_{i\in B} | \delta_i| \, \Big(\underbrace{\max_j \matNorm{X^B_{S_j}}_2}_{\text{(I)}}\Big) \, \underbrace{\Big({\textstyle \sum_{j} }\|[\errt]_{S_j} \|_2\Big)}_{\text{(II)}}
\end{align*}
where we use the fact that $\errw_i = 0$ if $i \in G$ and the Cauchy-Schwarz inequalities, and $X^B_{S_j}$ denotes $|B| \times |S_j|$ sub-matrix of $X^B \in \reals^{|B| \times p}$ corresponding only to indices $S_j$. 


\emph{(I):}
Provided $|B| \geq \exp(1)$, ${p \choose |B|} \leq \big(\frac{\exp(1) p}{|B|}\big)^{|B|} \leq p^{|B|}$. As discussed in \citet{Vershynin11,Nguyen2011b}, for every $t > 0$,
\begin{align*}
		\frac{1}{\sqrt{|B|}}\max_j \matNorm{X^B_{S_j}}_2 \leq \sqrt{\matNorm{\SigB}_2} \left(2 + t \right)  
\end{align*}
with probability at least $1 - 2{p \choose |B|} \exp (-t^2 |B|/2) \geq 1- 2\exp (-t^2|B|/2 + |B|\log p \big)$. Setting $t = 2\sqrt{\log p}$, we have
\begin{align*}
		\max_j \matNorm{X^B_{S_j}}_2 \leq 2(1+\sqrt{\log p}) \sqrt{\matNorm{\SigB}_2} \sqrt{|B|}
\end{align*}
with probability $1 - p^{-|B|}$. In the proof of Corollary \ref{Cor:XCor}, $\max_j \matNorm{X^B_{S_j}}_2 \leq f(X^B) \sqrt{|B|\log p}$ by assumption \ref{ConLS:XCor}, and the remaining proof would be exactly the same.

\emph{(II):} by the standard bound in \citet{CanRomTao06}, we obtain 
\begin{align*}
	\sum_{j} \|[\errt]_{S_j} \|_2 = \|[\errt]_{S_1} \|_2 + \sum_{j=2}^q \|[\errt]_{S_j} \|_2  \leq \|[\errt]_{S_1} \|_2 + \frac{1}{\sqrt{|B|}} \sum_{j=2}^q \|[\errt]_{S_{j}} \|_1 \leq \|\errt \|_2 + \frac{1}{\sqrt{|B|}} \|\errt \|_1 \, .
\end{align*}

Combining all pieces together yields 
\begin{align*}
	&\frac{1}{h}\sum_{i=1}^n \errw_i \big(\inner{x_i}{\Ttheta+\errt} - y_i \big) \inner{x_i}{\errt} \geq  - \frac{1}{h}\sum_{i=1}^n \errw_i \big(\epsilon_i + \delta_i\big) \inner{x_i}{\errt} \nonumber\\
	\geq \, & - 4\sqrt{\log p} \sqrt{\matNorm{\SigB}_2} \, \max_{i\in B} | \delta_i| \, \frac{|B|}{h}  \, \Big(\|\errt \|_2 + \frac{1}{\sqrt{|B|}} \|\errt \|_1\Big) \, ,
	%
\end{align*}
hence, we can guarantee \eqref{EqnLSInc} with functions
\begin{align*}
	&\IncConOne = 4 \sqrt{\matNorm{\SigB}_2} \max_{i\in B} | \delta_i| \sqrt{\log p}\frac{|B|}{h} \, , \\
	&\IncConTwo = 4 \sqrt{\matNorm{\SigB}_2} \max_{i\in B} | \delta_i| \sqrt{\log p} \frac{\sqrt{|B|}}{h} \, .
\end{align*}

To complete the proof, we need to specify the quantity $\big\| \frac{1}{h}\sum_{i=1}^n \Tw_i \big(\inner{x_i}{\Ttheta} - y_i \big) x_i \big\|_\infty = \big\| \frac{1}{h}\sum_{i\in G} \Tw_i \epsilon_i x_i \big\|_\infty$ for the appropriate choice of $\lambda$ as stated in Theorem \ref{ThmGeneral}.
By the sub-Gaussian property of noise vector $\epsilon$ in \ref{ConLS:NsubGauss}: 
for any fixed vector $v$ such that $\|v\|_2 = 1$, 
\begin{align*}
	\P\Big[|\inner{v}{\epsilon}| \geq t\Big] \leq 2 \exp \Big(-\frac{t^2}{2\sigma^2}\Big) \quad \text{for all } t >0 \, .
\end{align*}
Given vector $x_i$, let $x_i^j$ be the $j$-th element of vector $x_i$. Using the column normalization condition \ref{ConLS:colNorm} with $0 \leq \Tw \leq 1$, we have for all $j = 1,\hdots,p$  
\begin{align*}
	\P\bigg[\Big|\frac{1}{h}\sum_{i\in G} \Tw_i  x_i^j \epsilon_i \Big| \geq t\bigg] \leq 2 \exp \Big(-\frac{ht^2}{2\sigma^2}\Big) \quad \text{for all } t >0 \, ,
\end{align*}
and consequently by the union bound over,
\begin{align*}
	\P\bigg[\Big\|\frac{1}{h}\sum_{i\in G} \Tw_i  x_i^j \epsilon_i \Big\|_\infty \geq t\bigg] \leq 2 \exp \Big(-\frac{ht^2}{2\sigma^2}+ \log p\Big) \quad \text{for all } t >0 \, .
\end{align*}
Setting $t^2 = \frac{4\sigma^2\log p}{h}$, we obtain $\big\| \frac{1}{h}\sum_{i\in G} \Tw_i \epsilon_i x_i \big\|_\infty \leq \sqrt{\frac{4\sigma^2\log p}{h}}$ with probability at least $1 - c \exp(-c' h \lambda^2)$.

Now, we have all pieces to utilize Theorem \ref{ThmGeneral}. 
The assumption on choosing $\lambda$ in the statement is satisfied as follows:
\begin{align*}
	& 2\rho\RSCtolOne + \IncConTwo = 2\rho \frac{\kappa_2(2h-n)\log p}{2h^2} + 4 \sqrt{\matNorm{\SigB}_2} \max_{i\in B} | \delta_i| \sqrt{\log p} \frac{\sqrt{|B|}}{h} \nonumber\\
	\leq \, & C_1 \sqrt{\frac{h}{\log p}} \frac{\kappa_2(2h-n)\log p}{2h^2} + 4 \sqrt{\matNorm{\SigB}_2} C_2\sqrt{\frac{ h}{|B|}} \sqrt{\log p} \frac{\sqrt{|B|}}{h}\leq \Big(\frac{1}{2} C_1 \kappa_2 + 4 C_2\sqrt{\matNorm{\SigB}_2} \Big) \sqrt{\frac{\log p}{h}}
\end{align*}
where $C_2$ is some constant satisfying $C_2^2 \geq \frac{(\max_{i}\delta_i^2) |B|}{h}$, and we use the condition \ref{ConLS:R}. Finally, the RSC constant in \eqref{EqnLSrscConst} can be simply lower bounded with the assumption \ref{Con:h}:
\begin{align*}
	\frac{\kappa_1(2h-n)^2 }{4 h^2} \geq \kappa_1 \frac{\alpha^2}{4} \, ,
\end{align*}
hence we can have the bounds as stated.



\section{Results for Trimmed Graphical Lasso}
\subsection{Useful lemma(s)}

\begin{lemma}[Lemma 1 of \citet{RWRY11}]\label{Lem:SampleCov}
	Suppose that $\{\xi\}_{i=1}^n$ are iid samples from $N(0,\Sigma)$ with $n \geq 40 \max_i \Sigma_{ii}$. Let $\eventOne$ be the event that 
	\begin{align*}
		\bigg\| \frac{1}{n}\sum_{i=1}^n \xi(\xi)^\top - \Sigma \bigg\|_\infty \leq 8 (\max_i \Sigma_{ii}) \sqrt{\frac{10 \tau \log p}{n}}
	\end{align*}
	where $\tau$ is any constant greater than 2.
	Then, the probability of event $\eventOne$ occurring is at least $1- 4/p^{\tau -2}$. 
\end{lemma}

\begin{lemma}[Section B.4 of \citet{Loh13}]\label{Lem:RSC}
	For any $\errt \in \reals^{p \times p}$ such that $\|\errt\|_\F \leq 1$, 
	\begin{align*}
		\BigdoubleInner{ \big(\TParam\big)^{-1} - \big(\TParam + \errt\big)^{-1} }{\errt} \geq \big(\matNorm{\TParam}_2 + 1\big)^{-2} \|\errt\|_\F^2 \, .
	\end{align*}
\end{lemma}

\subsection{Proof of Corollary \ref{Cor1}}

Although Theorem \ref{ThmGeneral} can be seamlessly applied for the Trimmed Graphical Lasso as well, we need to restrict our attention to the case of $\|\errt\|_\F \leq 1$ in order to guarantee the (vanilla) restricted strong convex in Lemma \ref{Lem:RSC} (which is the standard technique even for the case without outliers as developed in \citet{Loh13}). Toward this, we first show that $\|\errt\|_\F \leq 1$ actually holds under the conditions :

\begin{lemma}\label{Lem:BoundDelta}
		Suppose that the condition \ref{Con:inc} holds. Moreover, $4\max\big\{\|\frac{1}{h}\sum_{i=1}^n \Tw_i \xi(\xi)^\top - (\TParam )^{-1}\|_\infty, \, \IncConTwo \big\} \leq \lambda \leq \frac{\RSCcon - \IncConOne}{3R}$. Then, for $(\LParam,\Lw)$, $\|\Lerrt\|_\F \leq 1$.
	\end{lemma}
	\begin{proof}
		The Lemma \ref{Lem:BoundDelta} can be proved by the fact $-\log\det{\Param}$ is a convex function. Hence, the function $f: [0,1] \rightarrow \reals$ given by $f(t;\TParam,\Lerrt) := -\log\det\big(\TParam + t\Lerrt\big)$ is also convex in $t$, and $\bigdoubleInner{-(\TParam + \Lerrt)^{-1}}{\Lerrt} \geq \bigdoubleInner{-(\TParam + t\Lerrt)^{-1}}{\Lerrt}$ for $t \in [0,1]$ (see \citet{Loh13} for details). 
		
		Now, suppose that $\|\Lerrt\|_\F \geq 1$. Then, we have
		\begin{align}\label{EqnLem1Tmp1}
			& \BigdoubleInner{ \big(\TParam\big)^{-1} - \big(\TParam + \Lerrt\big)^{-1} }{\Lerrt} \geq \BigdoubleInner{ \big(\TParam\big)^{-1} - \big(\TParam + t\Lerrt\big)^{-1} }{\Lerrt} \nonumber\\
			= \ & \frac{1}{t} \BigdoubleInner{ \big(\TParam\big)^{-1} - \big(\TParam + t\Lerrt\big)^{-1} }{t \Lerrt} \, .
		\end{align}
		Since $\|\Lerrt\|_\F \geq 1$, we can set $t = \frac{1}{\|\Lerrt\|_\F} \leq 1$ so that $\| t \Lerrt\|_\F = 1$. Hence, by applying Lemma \ref{Lem:RSC} for $t \Lerrt$, we obtain
		\begin{align*}
			\BigdoubleInner{ \big(\TParam\big)^{-1} - \big(\TParam + t\Lerrt\big)^{-1} }{t \Lerrt} \geq \RSCcon \|t\Lerrt\|_\F^2 = \RSCcon \, .
		\end{align*}
		Combining with \eqref{EqnLem1Tmp1} yields
		\begin{align}\label{EqnLem1Tmp2}
			\BigdoubleInner{ \big(\TParam\big)^{-1} - \big(\TParam + \Lerrt\big)^{-1} }{\Lerrt} \geq \RSCcon \|\Lerrt\|_\F \, .
		\end{align}
		 
		Now, from \eqref{EqnLocalMinima} and \eqref{EqnLem1Tmp2} followed by the condition \ref{Con:inc} and H\"older's inequity, we can obtain
		\begin{align*}
			& \RSCcon \|\Lerrt\|_\F  \leq \BigdoubleInner{ (\TParam )^{-1} - \frac{1}{h}\sum_{i=1}^n \Lw_i \xi(\xi)^\top}{\Lerrt} + \lambda(\offNorm{\Lerrt_{S}} - \offNorm{\Lerrt_{S^c}} ) \nonumber\\
			\leq \ &   \BigdoubleInner{ (\TParam )^{-1} - \frac{1}{h}\sum_{i=1}^n \Lw_i \xi(\xi)^\top}{\Lerrt} + \lambda\offNorm{\Lerrt} \nonumber\\
			\leq \ &  \BigdoubleInner{ (\TParam )^{-1} - \frac{1}{h}\sum_{i=1}^n \Tw_i \xi(\xi)^\top  + \frac{1}{h}\sum_{i=1}^n \Tw_i \xi(\xi)^\top - \frac{1}{h}\sum_{i=1}^n \Lw_i \xi(\xi)^\top}{\Lerrt} + \lambda\offNorm{\Lerrt} \nonumber\\
			\leq \ & \Big\|\frac{1}{h}\sum_{i=1}^n \Tw_i \xi(\xi)^\top - (\TParam )^{-1} \Big\|_{\infty} \, \cdot \|\Lerrt\|_1 + \IncConOne \|\Lerrt\|_\F + \IncConTwo \|\Lerrt\|_1 +  \lambda\offNorm{\Lerrt} \, .
		\end{align*}
		By the choice of $\lambda$ in the assumption of the statement and by the fact that $\offNorm{\Lerrt} \leq \|\Lerrt\|_1$ and $\|\Lerrt\|_1 \leq \|\LParam\|_1 + \|\TParam\|_1 \leq 2R$, we can rearrange the above inequality into
		\begin{align*}
			\|\Lerrt\|_\F \leq \frac{3 \lambda}{2\big(\RSCcon-\IncConOne\big)}  \|\Lerrt\|_1 \leq \frac{3 \lambda R}{\big(\RSCcon-\IncConOne\big)} \leq 1 \, ,
		\end{align*}
		which conflicts with the assumption in the beginning of this proof. Hence, by contradiction, we can conclude $\|\Lerrt\|_\F \leq 1$ under conditions in the statement.
	\end{proof}

	Since for this particular example, the modified restricted strong convexity condition in \ref{Con:rsc} is identical as the vanilla case (which is already proved in Lemma \ref{Lem:RSC}), the only remaining to utilize Theorem \ref{ThmGeneral} is to specify the quantity $\IncConOne$ and $\IncConTwo$ in \ref{Con:inc}. Toward this, we follow similar strategy as in \citet{Nguyen2011b}: given $\errt$, we divide the index of $\errt$ into the disjoint exhaustive subsets $S_1, S_2, \hdots, S_q$ of size $|B|$ such that $S_1$ contains $|B|$ largest absolute elements in $\errt$, and so on. Then, we have
\begin{align*}
	&\Big| \BigdoubleInner{\sum_{i=1}^n \Lerrw_i \xi(\xi)^\top}{\Lerrt} \Big| = \Big| \BigdoubleInner{\sum_{i \in B} \Lerrw_i \xi(\xi)^\top}{\Lerrt} \Big| \nonumber\\
	= \, & \bigg| \sum_{j=1}^q\BigdoubleInner{\sum_{i \in B} \Lerrw_i \big[\xi(\xi)^\top\big]_{S_j}}{\big[\Lerrt\big]_{S_j}} \bigg| \leq \sum_{j=1}^q \Big| \BigdoubleInner{\sum_{i \in B} \Lerrw_i \big[\xi(\xi)^\top\big]_{S_j}}{\big[\Lerrt\big]_{S_j}} \Big| \, . 
\end{align*}
Let $D_{\Lerrw}$ be a $|B|\times |B|$ diagonal matrix whose $i$-th diagonal entry is $[D_{\Lerrw}]_{ii} := \Lerrw_i$. Let also $X^B$ is a $|B|\times p$ design matrix for samples in the set $B$. Finally $X^B_{S_j}$ denotes a $|B|\times |S_j|$ sub-matrix of $X^B$ whose columns are indexed by $S_j$. Then,
\begin{align*}
	&\sum_{j=1}^q \Big| \BigdoubleInner{\sum_{i \in B} \Lerrw_i \big[\xi(\xi)^\top\big]_{S_j}}{\big[\Lerrt\big]_{S_j}} \Big| = \sum_{j=1}^q \Big| \text{Trace}\Big( \big[\Lerrt\big]_{S_j}^\top [X^B_{S_j}]^\top D_{\Lerrw}X^B_{S_j} \Big) \Big| \nonumber\\
	= \, & \sum_{j=1}^q \Big| \BigdoubleInner{X^B_{S_j} \big[\Lerrt\big]_{S_j}}{D_{\Lerrw}X^B_{S_j}} \Big| \leq \sum_{j=1}^q \big\|X^B_{S_j} \big[\Lerrt\big]_{S_j} \big\|_\F \, \big\|D_{\Lerrw}X^B_{S_j}\big\|_\F \nonumber\\
	\leq \, & \sqrt{|B|}\Big(\max_j \matNorm{X^B_{S_j}}_2\Big)^2 \, \underbrace{\Big({\textstyle \sum_{j} }\|[\Lerrt]_{S_j} \|_\F\Big)}_{\text{(I)}} \, .
\end{align*}

\emph{(I):} by the standard bound in \citet{CanRomTao06}, we obtain 
\begin{align*}
	\sum_{j} \|[\Lerrt]_{S_j} \|_\F = \|[\Lerrt]_{S_1} \|_\F + \sum_{j=2}^q \|[\Lerrt]_{S_j} \|_1  \leq \|[\Lerrt]_{S_1} \|_\F + \frac{1}{\sqrt{|B|}} \sum_{j=2}^q \|[\Lerrt]_{S_{j}} \|_1 \leq \|\Lerrt \|_\F + \frac{1}{\sqrt{|B|}} \|\Lerrt \|_1 \, .
\end{align*}

Combining all pieces together yields 
\begin{align*}
	&\Big| \BigdoubleInner{\frac{1}{h}\sum_{i=1}^n \Lerrw_i \xi(\xi)^\top}{\Lerrt} \Big| \leq  \frac{\sqrt{|B|}}{h} \Big(\max_j \matNorm{X^B_{S_j}}_2\Big)^2 \, \bigg(\|\Lerrt \|_\F + \frac{1}{\sqrt{|B|}} \|\Lerrt \|_1 \bigg) \, ,
\end{align*}
hence, we can guarantee the condition \ref{Con:inc} with functions
\begin{align*}
	&\IncConOne = f(X^B) \sqrt{\frac{|B|\log p}{h}} \quad \text{and} \\ 
	&\IncConTwo = f(X^B) \sqrt{\frac{\log p}{h}}  \, . 
\end{align*}

To complete the proof, we also need to specify the quantity $\big\|\frac{1}{h}\sum_{i=1}^n \Tw_i \xi(\xi)^\top - (\TParam )^{-1}\big\|_{\infty} = \big\|\frac{1}{h}\sum_{i \in G} \Tw_i \xi(\xi)^\top - (\TParam )^{-1}\big\|_{\infty}$ for the appropriate choice of $\lambda$ as stated in Theorem \ref{ThmGeneral}. Recall that we constructed $\Tw$ as follows: $\Tw_i$ is simply set to $\Lw_i$ if $i \in G$, and $\Tw_i =0$ for $i \in B$. Let $\widebar{G}$ be the subset of $G$ such that $\Tw_i = 1$ and $\widebar{G}^c$ be the subset such that $\Tw_i = 0$. Then, we have $h \geq |\widebar{G}| \geq h - |B|$, and hence $h - |\widebar{G}| \leq |B|$. Now, by Lemma \ref{Lem:SampleCov}, we can obtain the following bound:
\begin{align*}
	& \Big\|\frac{1}{h}\sum_{i=1}^n \Tw_i \xi(\xi)^\top - (\TParam )^{-1}\Big\|_{\infty} = \bigg\|\frac{|\bar{G}|}{h}\frac{1}{|\widebar{G}|}\sum_{i \in \widebar{G}}  \xi(\xi)^\top - (\TParam )^{-1}\bigg\|_{\infty} \nonumber\\
	= \ & \bigg\|\frac{|\widebar{G}|}{h} \Big( \frac{1}{|\widebar{G}|}\sum_{i \in \bar{G}}  \xi(\xi)^\top - (\TParam )^{-1} \Big) - \Big( \frac{h - |\widebar{G}|}{h}\Big) (\TParam )^{-1}\bigg\|_{\infty} \nonumber\\
	\leq \ & \bigg\|\frac{|\widebar{G}|}{h} \Big( \frac{1}{|\widebar{G}|}\sum_{i \in \bar{G}}  \xi(\xi)^\top - (\TParam )^{-1} \Big)\bigg\|_\infty + \bigg\|\Big( \frac{h - |\widebar{G}|}{h}\Big) (\TParam )^{-1}\bigg\|_{\infty} \nonumber\\
	\leq \ & \bigg\| \frac{1}{|\widebar{G}|}\sum_{i \in \bar{G}}  \xi(\xi)^\top - (\TParam )^{-1} \bigg\|_\infty + \frac{|B|}{h}\| \Sigma^*\|_{\infty} \nonumber\\
	\leq \ & 8 (\max_i \Sigma^*_{ii}) \sqrt{\frac{10 \tau \log p }{\widebar{G}}} + \frac{|B|}{h}\| \Sigma^*\|_{\infty} \leq 8 (\max_i \Sigma^*_{ii}) \sqrt{\frac{10 \tau \log p }{h - |B|}} + \frac{|B|}{h}\| \Sigma^*\|_{\infty} 
\end{align*}
with probability at least $1- 4/p^{\tau -2}$ for any $\tau >2$.


\subsection{Proof of Corollary \ref{Cor2}}
Under \ref{Con:2}, $h-|B| \geq (n - a\sqrt{n}) - a\sqrt{n}$. Hence, if $n \geq 16 a^2$, then $h-|B| \geq (n - a\sqrt{n}) - a\sqrt{n} \geq \frac{n}{2}$. Moreover, $\frac{|B|}{h} \leq \frac{a\sqrt{n}}{n/2} \leq \frac{2a}{\sqrt{n}}$. Therefore, from the Corollary \ref{Cor1}, the selection of $\lambda$ in the statement satisfies $\lambda \geq 4\max\big\{ \|\frac{1}{h}\sum_{i=1}^n \Tw_i \xi(\xi)^\top - (\TParam )^{-1}\|_{\infty}\, , \, \IncConTwo \big\}$.

Furthermore, as long as $n \geq \big(\matNorm{\TParam}_2 + 1\big)^{4} \big(3Rc + f(X^B)\sqrt{2|B|} \big)^2(\log p)$,
\begin{align*}
	\lambda = c\sqrt{\frac{\log p}{n}} \leq \frac{\big(\matNorm{\TParam}_2 + 1\big)^{-2} - f(X^B)\sqrt{\frac{2|B| \log p}{n}} }{3R} \, ,
\end{align*}
and therefore we have $\lambda \leq \frac{\RSCcon - f(X^B) \sqrt{\frac{|B|\log p}{h}}}{3R}$ where $c$ is defined as $4\max \big\{16 (\max_i \Sigma^*_{ii}) \sqrt{5 \tau} + \frac{2a\| \Sigma^*\|_{\infty}}{\sqrt{\log p}} \, , \, \sqrt{2} f(X^B) \big\}$, as stated.


\subsection{Proof of Corollary \ref{Cor3}}
In this proof, we simply need to specify the quantity $f(X^B)$ under the condition \ref{Con:2} and \ref{Con:3}, and then we can appeal to the result in Corollary \ref{Cor2}.  

Provided $|B| \geq \exp(1)$, ${p \choose |B|} \leq \big(\frac{\exp(1) p}{|B|}\big)^{|B|} \leq p^{|B|}$. As discussed in \citet{Vershynin11,Nguyen2011b}, if \ref{Con:3} holds, for every $t > 0$, we have
\begin{align*}
		\frac{1}{\sqrt{|B|}}\max_j \matNorm{X^B_{S_j}}_2 \leq \sqrt{\matNorm{\SigB}_2} \left(2 + t \right)  
\end{align*}
with probability at least $1 - 2{p \choose |B|} \exp (-t^2 |B|/2) \geq 1- 2\exp (-t^2|B|/2 + |B|\log p \big)$. Setting $t = 2\sqrt{\log p}$, we obtain
\begin{align*}
		\max_j \matNorm{X^B_{S_j}}_2 \leq 2(1+\sqrt{\log p}) \sqrt{\matNorm{\SigB}_2} \sqrt{|B|}
\end{align*}
with probability $1 - p^{-|B|}$. 
Therefore, under \ref{Con:2}, 
\begin{align*}
	& \Big(\max_j \matNorm{X^B_{S_j}}_2\Big)^2 \leq 4\big(1+\sqrt{\log p}\big)^2 \matNorm{\SigB}_2 |B| \leq 4 a \big(1+\sqrt{\log p}\big)^2 \matNorm{\SigB}_2  \sqrt{n} \nonumber\\
	= \ & \frac{4 a \big(1+\sqrt{\log p}\big)^2 \matNorm{\SigB}_2  \sqrt{n}}{\sqrt{h\log p}} \sqrt{h\log p} \leq \frac{4\sqrt{2} a \big(1+\sqrt{\log p}\big)^2 \matNorm{\SigB}_2}{\sqrt{\log p}} \sqrt{h\log p} \, ,
\end{align*} 
as specified in the statement.

\section{Proof of Proposition \ref{ThmOpt}}

Since we assume $\{\theta^{(t)}\}$ converges for any fixed $w$, $f(\theta^{(t)};w)$ monotonically decreases in $t$:
\begin{align*}
	f(\theta^{(t+1)};w) - f(\theta^{(t)};w) \leq 0 \, .
\end{align*} 
Setting $w = w^{(t)}$ above, we have 
\begin{align}\label{EqnThmOptTmp1}
	f(\theta^{(t+1)};w^{(t)})  \leq  f(\theta^{(t)};w^{(t)}) \, .
\end{align}
Since $w^{(t+1)}$ is computed to minimize $\min_w f(\theta^{(t+1)};w)$, it holds
\begin{align}\label{EqnThmOptTmp2}
	f(\theta^{(t+1)};w^{(t+1)}) \leq f(\theta^{(t+1)};w^{(t)}) \, .
\end{align}
By combining \eqref{EqnThmOptTmp1} and \eqref{EqnThmOptTmp2}, we obtain 
\begin{equation}
\label{eq:descent}
f(\theta^{(t+1)};w^{(t+1)}) \leq f(\theta^{(t)};w^{(t)})\quad \mbox{for all}\quad t,
\end{equation}
establishing monotonic decrease of function values. Since the domain of $F$ is compact, we know a limit point exists. 

Next, we can take each $w^{(t)}$ to be a vertex of the capped simplex, since the subproblem for $w$ is a linear program 
(indeed, our implementation only chooses vertex solutions $w^{(t)}$). 
Therefore, along a subsequence $t_k$ that converges to any limit point $(v, \overline \theta)$,
the weights $w^{t_k}$ converge to $v$ after finitely many steps (since all vertices are separated by 
some positive distance). Once $w^{(t_k)}$ have converged $v$, iterates in the extended 
framework are identical to those generated by Algorithm 
$\mathcal{A}$ for the associated data selection, and therefore $\overline\theta$ 
is a stationary point for the associated $M$-estimator.  
Then $(v,\overline\theta)$ is a stationary point for the overall problem. 

Suppose now that two limit points correspond to two different vertices $v_1$ and $v_2$.
Each vertex of the capped simplex corresponds to a selection of data points, which we call $\mathcal{D}_1$ and $\mathcal{D}_2$. Consider subsequences $t_{k_1}$ and $t_{k_2}$ which converge to $v_1$ and $v_2$, respectively. 
Along each subsequence, $w^{(t_{k_i})}$ converge to $v_i$ after finitely many steps as discussed above, 
and again the iterates of the extended algorithm are identical to those
generated by $\mathcal{A}$ for the M-estimators defined over $\mathcal{D}_1$
and $\mathcal{D}_2$.

In order to make a stronger statement, we need to make stronger assumptions. Suppose that 
\begin{enumerate}
\item the original M-estimator is convex, and 
\item the optimization problems over each vertex $v_k$ (corresponding to 
data selection $\mathcal{D}_k$) have different optimal values. 
\end{enumerate}
Then there exists an $\epsilon > 0$
so that without loss of generality, $f(\theta_1^*; v_1) + \epsilon \leq f(\theta_1^*; v_2)$. 
Now, since each problem is convex over its respective dataset, we can guarantee that after  $k\geq T$ steps 
of $\mathcal{A}$ along the subsequence $t_{k_1}$,  we have $f(\theta^{t_{k_1}}; v_1) < f(\theta_1^*; v_1) + \frac{\epsilon}{2} <f(\theta_1^*; v_2)$, 
and it is impossible for the algorithm to return to $v_2$ by the already established descent property~\eqref{eq:descent}. 
The number of iterations can be precisely quantified, see e.g.~\citet{nest_lect_intro}.

The contradiction ensures that the weights converge after finitely many steps to a single vertex $v$. Once the weights converge to $v$, we know that {\it all iterates} of the extended algorithm are identical to those of $\mathcal{A}$ for the convex problem defined over selection $\mathcal{D}$ associated to $v$, and the extended algorithm 
converges to a stationary point of the problem.

\end{document}